\documentclass{article}
\usepackage[latin9]{inputenc}
\usepackage{amsmath}
\usepackage{amsthm}
\usepackage{amssymb}
\usepackage{graphicx}
\usepackage{microtype}
\usepackage[unicode=true,
 bookmarks=false,
 breaklinks=false,pdfborder={0 0 1},backref=section,colorlinks=false]
 {hyperref}
\usepackage{breakurl}

\makeatletter
\theoremstyle{plain}
\newtheorem{thm}{\protect\theoremname}
\theoremstyle{plain}
\newtheorem*{thm*}{\protect\theoremname}

 \newcommand{\mat}[1]{\mathbf{#1}}
 \renewcommand{\vec}[1]{\mathbf{#1}}
 \newcommand{\D}{\mathrm{d}}
 \newcommand{\I}{\mathrm{i}}
 \newcommand{\E}{\mathrm{e}}

 \newcommand{\R}{\mathbb{R}}
 
 \newcommand{\N}{\mathbb{N}}

 \newcommand*\diff{\mathop{}\!\D}

\usepackage{iclr2020_conference}
\usepackage{times}

\usepackage{url}

\usepackage{changepage}

\usepackage{amsfonts}% blackboard math symbols
\theoremstyle{plain}
\providecommand{\theoremname}{Theorem}

\usepackage{lineno}
\usepackage{bbm}

\title{Frequency Principle: Fourier Analysis Sheds Light on Deep Neural Networks}

\author{
Zhi-Qin John Xu\thanks{Corresponding author, https://ins.sjtu.edu.cn/people/xuzhiqin/}\\
Shanghai Jiao Tong University\\
xuzhiqin@sjtu.edu.cn
\And
Yaoyu Zhang \\
Institute for Advanced Study\\
Yaoyu@ias.edu
\And 
  Tao Luo\\
Purdue University\\
luo196@purdue.edu
\And
Yanyang Xiao \\
Shenzhen Institutes of Advanced Technology\\
xyy82148@gmail.com
  \And 
  Zheng Ma\\
Purdue University\\
ma531@purdue.edu
}
\@ifundefined{showcaptionsetup}{}{%
 \PassOptionsToPackage{caption=false}{subfig}}
\usepackage{subfig}
\makeatother

\providecommand{\theoremname}{Theorem}

\begin{document}
\maketitle
\begin{abstract}
We study the training process of Deep Neural Networks (DNNs) from
the Fourier analysis perspective. We demonstrate a very universal Frequency
Principle (F-Principle) --- DNNs often fit target functions from
low to high frequencies --- on high-dimensional benchmark datasets such as MNIST/CIFAR10 and deep neural networks such as VGG16. This F-Principle of 
DNNs is opposite to the behavior of most conventional iterative numerical schemes (e.g., Jacobi method), which exhibit faster convergence for higher frequencies for various scientific computing problems. With a simple theory, we illustrate that this F-Principle results from the regularity of the commonly used activation functions. The F-Principle implies an implicit bias that DNNs tend to fit training data by a low-frequency function. This understanding provides an explanation of good generalization of DNNs on most real datasets and bad generalization of DNNs on parity function or randomized dataset.
\end{abstract}

\section{Introduction\label{sec:Introduction}}

Understanding the training process of Deep Neural Networks (DNNs)
is a fundamental problem in the area of deep learning. We find a
common implicit bias in the gradient-based training process of DNNs, that is,
a Frequency Principle (F-Principle):

\begin{changemargin}{0.5cm}{0.5cm}\emph{DNNs often fit target
functions from low to high frequencies during the training process.}
\end{changemargin}

In another word, at the early stage of training, the low-frequencies
are fitted and as iteration steps of training increase, the high-frequencies
are fitted. For example, when a DNN is trained to fit $y=\sin(x)+\sin(2x)$,
its output would be close to $\sin(x)$ at early stage and as training goes on, its output would be close to $\sin(x)+\sin(2x)$.
F-Principle was observed empirically in synthetic low-dimensional data with MSE loss during DNN training \citep{xu_training_2018,rahaman2018spectral}.
However, in deep learning, empirical phenomena vary from one network structure to another, from one dataset to another and often show significant difference between synthetic data and high-dimensional real data. It is still of great challenges to quantitatively study the universality of empirical observed phenomena, e.g., F-Principle, in high-dimensional real problems due to large computational cost, for instance, the high-dimensional Fourier transform is prohibitive in practice. In addition, it is also unclear whether the F-Principle can guide the usage and provide insight of DNNs in real problems. 

In this work, we first design two methods to show that the F-Principle
exists in the training process of DNNs for different benchmark setups,
e.g., MNIST \citep{lecun1998mnist}, CIFAR10 \citep{krizhevsky2010cifar},
and deep networks, such as VGG16 \citep{simonyan2014very}. The settings
we have considered are i) different DNN architectures, e.g., fully-connected
network and convolutional neural network (CNN); ii) different activation
functions, e.g., tanh and rectified linear unit (ReLU); iii) different
loss functions, e.g., cross entropy, mean squared error (MSE), and
loss energy functional in variation problems. These results demonstrate
the universality of the F-Principle.

To facilitate the designs and applications of DNN-based schemes, we
characterize a stark difference between DNNs and conventional numerical
schemes on various scientific computing problems, where most of
the conventional methods (e.g., Jacobi method) exhibit the opposite
convergence behavior --- faster convergence for higher frequencies.
This difference implicates that DNN can be adopted to accelerate the
convergence of low frequencies for computational problems.

We also show how the power decaying spectrum of commonly used activation
functions contributes to the F-Principle with a theory under an idealized
setting. Note that this mechanism is rigorously demonstrated for
DNNs of general settings in a subsequent work \citep{luo2019theory}.
Finally, we discuss that the F-Principle provides an understanding
of good generalization of DNNs in many real datasets \citep{zhang2016understanding}
and poor generalization in learning the parity function \citep{shalev2017failures,nye2018efficient},
that is, the F-Principle which implies that DNNs prefer low frequencies,
is consistent with the property of low frequencies dominance in many
real datasets, e.g., MNIST/CIFAR10, but is different from the parity
function whose spectrum concentrates on high frequencies. Compared with previous studies, which only study synthetic data with MSE loss, our main contributions are as follows:

1. By designing both the projection and filtering method, we consistently
demonstrate the F-Principle for high-dimensional real datasets of
MNIST/CIFAR10 over various architectures such as VGG16.

2. For the application of solving differential equations, we show (i) conventional numerical schemes learn higher frequencies
faster whereas DNNs learn lower frequencies faster by the F-Principle,
(ii) convergence of low frequencies can be greatly accelerated with
DNN-based schemes.

3. We show a simple theory under an idealized setting for an easy understanding
of the F-Principle.

4. We discuss in detail the impact of the F-Principle to the generalization
of DNNs that DNNs are implicitly biased towards
a low frequency function leading to good and poor generalization
for low and high frequency dominant target functions respectively.

\section{Frequency Principle}

The concept of ``frequency'' is central to the understanding of
F-Principle. In this paper, \emph{the ``frequency'' means }\textbf{\emph{response
frequency}} NOT image (or input) frequency as explained in the following.

Image (or input) frequency (NOT used in the paper): Frequency of $2$-d
function $I:\mathbb{R}^{2}\to\mathbb{R}$ representing the intensity
of an image over pixels at different locations. This frequency corresponds
to the rate of change of intensity \emph{across neighbouring pixels}.
For example, an image of constant intensity possesses only the zero
frequency, i.e., the lowest frequency, while a sharp edge contributes
to high frequencies of the image.

\textbf{Response frequency} (used in the paper): Frequency of a general
Input-Output mapping $f$. For example, consider a simplified classification
problem of partial MNIST data using only the data with label $0$
and $1$, $f(x_{1},x_{2},\cdots,x_{784}):\mathbb{R}^{784}\to\{0,1\}$
mapping $784$-d space of pixel values to $1$-d space, where $x_{i}$
is the intensity of the $i$-th pixel. Denote the mapping's Fourier
transform as $\hat{f}(k_{1},k_{2},\cdots,k_{784})$. The frequency
in the coordinate $k_{i}$ measures the rate of change of $f(x_{1},x_{2},\cdots,x_{784})$
\emph{with respect to $x_{i}$, i.e., the intensity of the $i$-th
pixel}. If $f$ possesses significant high frequencies for large $k_{i}$,
then a small change of $x_{i}$  in the image
might induce a large change of the output (e.g., adversarial example).
For a dataset with multiple classes, we can similarly define frequency
for each output dimension. An illustration of F-Principle using a
function of 1-d input is in Appendix \ref{sec:1dexp}.

\textbf{Frequency Principle}: \emph{DNNs often fit target functions from low
to high (response) frequencies during the training process.} In the following, by using high-dimensional real datasets, we experimentally demonstrate F-Principle at the levels of both individual frequencies (projection
method) and coarse-grained frequencies (filtering method). 

\section{F-Principle in MNIST/CIFAR10 through projection method \label{sec:project}}

Real datasets are very different from synthetic data used in previous
studies. In order to utilize the F-Principle to understand and better use DNNs
in real datasets, it is important to verify whether the F-Principle
also holds in high-dimensional real datasets.

In the following experiments, we examine the F-Principle in a training
dataset of $\{(\vec{x}_{i},\vec{y}_{i})\}_{i=0}^{n-1}$ where
$n$ is the size of dataset. $\vec{x}_{i}\in\R^{d}$
is a vector representing the image and $\vec{y}_{i}\in\{0,1\}^{10}$
is the output (a one-hot vector indicating the label for the dataset
of image classification). $d$ is the dimension of the
input ($d=784$ for MNIST and $d=32\times32\times3$
for CIFAR10). Since the high dimensional discrete Fourier transform (DFT)
has high computational cost, in this section, we only consider one
direction in the Fourier space through a projection method for each
examination.

\subsection{Examination method: Projection}

For a dataset $\{(\vec{x}_{i},\vec{y}_{i})\}_{i=0}^{n-1}$ we consider one entry of 10-d output, denoted by
$y_{i}\in\mathbb{R}$. The high dimensional discrete non-uniform Fourier
transform of $\{(\vec{x}_{i},y_{i})\}_{i=0}^{n-1}$ is $\hat{y}_{\vec{k}}=\frac{1}{n}\sum_{i=0}^{n-1}y_{i}\exp\left(-\I 2\pi\vec{k}\cdot\vec{x}_{i}\right)$.
The number of all possible $\vec{k}$ grows exponentially on dimension
$d$. For illustration, in each examination, we
consider a direction of $\vec{k}$ in the Fourier space, i.e.,
$\vec{k}=k\vec{p}_{1}$, $\vec{p}_{1}$
is a chosen and fixed unit vector, hence $|\vec{k}|=k$. Then
we have $\hat{y}_{k}=\frac{1}{n}\sum_{i=0}^{n-1}y_{i}\exp\left(-\I2\pi(\vec{p}_{1}\cdot\vec{x}_{j})k\right)$,
which is essentially the $1$-d Fourier transform of $\{(x_{\vec{p}_{1},i},y_{i})\}_{i=0}^{n-1}$,
where $x_{\vec{p}_{1},i}=\vec{p}_{1}\cdot\vec{x}_{i}$ is
the projection of $\vec{x}_{i}$ on the direction $\vec{p}_{1}$
\citep{bracewell1986fourier}. For each training dataset, $\vec{p}_{1}$
is chosen as the first principle component of the input space. To
examine the convergence behavior of different frequency components
during the training, we compute the relative difference between the DNN
output and the target function for selected important frequencies
$k$'s at each recording step, that is, $\Delta_{F}(k)=|\hat{h}_{k}-\hat{y}_{k}|/|\hat{y}_{k}|$,
where $\hat{y}_{k}$ and $\hat{h}_{k}$ are $1$-d Fourier transforms
of $\{y_{i}\}_{i=0}^{n-1}$ and the corresponding DNN output$\{h_{i}\}_{i=0}^{n-1}$,
respectively, along $\vec{p}_{1}$.

\subsection{MNIST/CIFAR10}

In the following, we show empirically that the F-Principle is exhibited
in the selected direction during the training process of DNNs when
applied to MNIST/CIFAR10 with cross-entropy loss. The network for
MNIST is a fully-connected tanh DNN ($784$-$400$-$200$-$10$) and for CIFAR10
is two ReLU convolutional layers followed by a fully-connected DNN
($800$-$400$-$400$-$400$-$10$). All experimental details of this paper can be
found in Appendix~\ref{sec:Experimental-settings}. We consider one of the 10-d outputs in each case using non-uniform
Fourier transform, as shown in Fig.~\ref{fig:CMFT}(a) and \ref{fig:CMFT}(c),
low frequencies dominate in both real datasets. During the training,
the evolution of relative errors of certain selected frequencies (marked
by black squares in Fig.~\ref{fig:CMFT}(a) and \ref{fig:CMFT}(c))
is shown in Fig.~\ref{fig:CMFT}(b) and \ref{fig:CMFT}(d). One can
easily observe that DNNs capture low frequencies first and gradually
capture higher frequencies. Clearly, this behavior is consistent with
the F-Principle. For other components of the output vector and other
directions of $\vec{p}$, similar phenomena are also observed. 
\begin{center}
\begin{figure*}[h]
\begin{centering}
\includegraphics[scale=0.6]{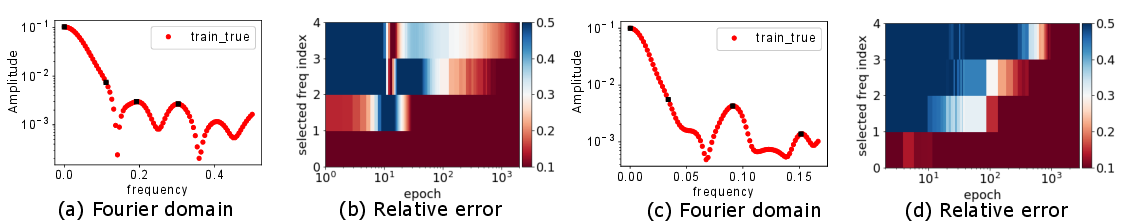} 
\par\end{centering}
\caption{Projection method. (a, b) are for MNIST, (c, d) for CIFAR10. (a, c)
Amplitude $|\hat{y}_{k}|$ vs. frequency. Selected frequencies are
marked by black squares. (b, d) $\Delta_{F}(k)$ vs. training epochs
for the selected frequencies. \label{fig:CMFT} }
\end{figure*}
\par\end{center}

\section{F-Principle in MNIST/CIFAR10 through filtering method \label{sec:Filter}}

The projection method in previous section enables us to visualize
the F-Principle in one direction for each examination at the level
of individual frequency components. However, this demonstration of
the F-Principle is insufficient since it is impossible to verify F-Principle
at all potentially informative directions for high-dimensional data.
To compensate the projection method, in this section, we consider
a coarse-grained filtering method which is able to unravel whether,
in the radially averaged sense, low frequencies converge faster than
high frequencies.

\subsection{Examination method: filtering \label{subsec:Filtering-method}}

% We decompose the frequency space into two domains
We split the frequency domain into two parts, i.e., a low-frequency
part with $|\vec{k}|\leq k_{0}$ and a high-frequency part with
$|\vec{k}|>k_{0}$, where $|\cdot|$ is the length of a vector.
The DNN is trained as usual by the original dataset $\{(\vec{x}_{i},\vec{y}_{i})\}_{i=0}^{n-1}$,
such as MNIST or CIFAR10. The DNN output is denoted as $\vec{h}$.
During the training, we can examine the convergence of relative errors
of low- and high- frequency part, using the two measures below

\begin{equation*}
    e_{\mathrm{low}}=\left(\frac{\sum_{\vec{k}}\mathbbm{1}_{|\vec{k}|\leq k_{0}}|\vec{\hat{y}}(\vec{k})-\vec{\hat{h}}(\vec{k})|^{2}}{\sum_{\vec{k}}\mathbbm{1}_{|\vec{k}|\leq k_{0}}|\vec{\hat{y}}(\vec{k})|^{2}}\right)^{\frac{1}{2}},\quad e_{\mathrm{high}}=\left(\frac{\sum_{\vec{k}}(1-\mathbbm{1}_{|\vec{k}|\leq k_{0}})|\vec{\hat{y}}(\vec{k})-\vec{\hat{h}}(\vec{k})|^{2}}{\sum_{\vec{k}}(1-\mathbbm{1}_{|\vec{k}|\leq k_{0}})|\vec{\hat{y}}(\vec{k})|^{2}}\right)^{\frac{1}{2}},
\end{equation*}
respectively, where $\hat{\cdot}$ indicates Fourier transform, $\mathbbm{1}_{\vec{k}\leq k_{0}}$
is an indicator function, i.e., 
\begin{equation*}
    \mathbbm{1}_{|\vec{k}|\leq k_{0}}=\begin{cases}
    1, & |\vec{k}|\leq k_{0},\\
    0, & |\vec{k}|>k_{0}.
    \end{cases}
\end{equation*}
If we consistently observe $e_{\mathrm{low}}<e_{\mathrm{high}}$ for different
$k_{0}$'s during the training, then in a mean sense, lower frequencies
are first captured by the DNN, i.e., F-Principle.

However, because it is almost impossible to compute above quantities
numerically due to high computational cost of high-dimensional Fourier
transform, we alternatively use the Fourier transform of a Gaussian
function $\hat{G}^{\delta}(\vec{k})$, where $\delta$ is the variance
of the Gaussian function $G$, to approximate $\mathbbm{1}_{|\vec{k}|>k_{0}}$.
This is reasonable due to the following two reasons. First, the Fourier
transform of a Gaussian is still a Gaussian, i.e., $\hat{G}^{\delta}(\vec{k})$
decays exponentially as $|\vec{k}|$ increases, therefore, it can
approximate $\mathbbm{1}_{|\vec{k}|\leq k_{0}}$ by $\hat{G}^{\delta}(\vec{k})$
with a proper $\delta(k_{0})$ (referred to as $\delta$ for simplicity).
Second, the computation of $e_{\mathrm{low}}$ and $e_{\mathrm{high}}$
contains the multiplication of Fourier transforms in the frequency
domain, which is equivalent to the Fourier transform of a convolution
in the spatial domain. We can equivalently perform the examination
in the spatial domain so as to avoid the almost impossible high-dimensional
Fourier transform. The low frequency part can be derived by 
\begin{equation}
    \vec{y}_{i}^{\mathrm{low},\delta}\triangleq(\vec{y}*G^{\delta})_{i},\label{eq:filter-1}
\end{equation}
where $*$ indicates convolution operator, and the high frequency
part can be derived by 
\begin{equation}
    \vec{y}_{i}^{\mathrm{high},\delta}\triangleq\vec{y}_{i}-\vec{y}_{i}^{\mathrm{low},\delta}.\label{eq:yleft-1}
\end{equation}
Then, we can examine 
\begin{equation}
    e_{\mathrm{low}}=\left(\frac{\sum_{i}|\vec{y}_{i}^{\mathrm{low},\delta}-\vec{h}_{i}^{\mathrm{low},\delta}|^{2}}{\sum_{i}|\vec{y}_{i}^{\mathrm{low},\delta}|^{2}}\right)^{\frac{1}{2}},\quad e_{\mathrm{high}}=\left(\frac{\sum_{i}|\vec{y}_{i}^{\mathrm{high},\delta}-\vec{h}_{i}^{\mathrm{high},\delta}|^{2}}{\sum_{i}|\vec{y}_{i}^{\mathrm{high},\delta}|^{2}}\right)^{\frac{1}{2}},\label{eq:ehigh}
\end{equation}
where $\vec{h}^{\mathrm{low},\delta}$ and $\vec{h}^{\mathrm{high},\delta}$
are obtained from the DNN output $\vec{h}$ through the same decomposition.
If $e_{\mathrm{low}}<e_{\mathrm{high}}$ for different $\delta$'s during
the training, F-Principle holds; otherwise, it is falsified. Next,
we introduce the experimental procedure.

Step One: \textbf{Training}. Train the DNN by the \emph{original dataset}
$\{(\vec{x}_{i},\vec{y}_{i})\}_{i=0}^{n-1}$, such as MNIST or
CIFAR10. $\vec{x}_{i}$ is an image vector, $\vec{y}_{i}$ is
a one-hot vector.

Step Two: \textbf{Filtering}. The low frequency part can be derived
by 
\begin{equation}
    \vec{y}_{i}^{{\rm low},\delta}=\frac{1}{C_{i}}\sum_{j=0}^{n-1}\vec{y}_{j}G^{\delta}(\vec{x}_{i}-\vec{x}_{j}),\label{eq:filter}
\end{equation}
where $C_{i}=\sum_{j=0}^{n-1}G^{\delta}(\vec{x}_{i}-\vec{x}_{j})$
is a normalization factor and 
\begin{equation}
    G^{\delta}(\vec{x}_{i}-\vec{x}_{j})=\exp\left(-|\vec{x}_{i}-\vec{x}_{j}|^{2}/(2\delta)\right).
\end{equation}
The high frequency part can be derived by $\vec{y}_{i}^{\mathrm{high},\delta}\triangleq\vec{y}_{i}-\vec{y}_{i}^{\mathrm{low},\delta}$.
We also compute $\vec{h}_{i}^{\mathrm{low},\delta}$ and $\vec{h}_{i}^{\mathrm{high},\delta}$
for each DNN output $\vec{h}_{i}$.

Step Three: \textbf{Examination}. To quantify the convergence of $\vec{h}^{\mathrm{low},\delta}$
and $\vec{h}^{\mathrm{high},\delta}$, we compute the relative error
$e_{\mathrm{low}}$ and $e_{\mathrm{high}}$ through Eq. (\ref{eq:ehigh}).

\subsection{DNNs with various settings}

It is important to verify the F-Principle in commonly used large
networks. With the filtering method, we show the F-Principle in the
DNN training process of real datasets. For MNIST, we use a fully-connected
tanh-DNN (no softmax) with MSE loss; for CIFAR10, we use cross-entropy
loss and two structures, one is small ReLU-CNN network, i.e., two convolutional layers,
followed by a fully-connected multi-layer neural network with a softmax; the other is
VGG16 \citep{simonyan2014very} equipped with a 1024 fully-connected
layer. These three structures are denoted as ``DNN'', ``CNN''
and ``VGG'' in Fig. \ref{fig:Noisefitting-Mnist}, respectively.
All are trained by SGD from \emph{scratch}. More details are in Appendix~\ref{sec:Experimental-settings}

We scan a large range of $\delta$ for both datasets, as an example,
results of each dataset for  several $\delta$'s are shown in Fig. \ref{fig:Noisefitting-Mnist}, respectively. Red
color indicates small relative error. In all cases, the relative error
of the low-frequency part, i.e., $e_{\mathrm{low}}$, decreases (turns
red) much faster than that of the high-frequency part, i.e., $e_{\mathrm{high}}$.
As analyzed above, the low-frequency part converges faster than the
high-frequency part. We also remark that, based on above the results
on cross-entropy loss, the F-Principle is not limited to MSE loss,
which possesses a natural Fourier domain interpretation by the Parseval's
theorem. Note that the above results holds for both SGD and GD. 
\begin{center}
\begin{figure}[h]
\begin{centering}
\subfloat[$\delta=3$, DNN]{\begin{centering}
\includegraphics[scale=0.22]{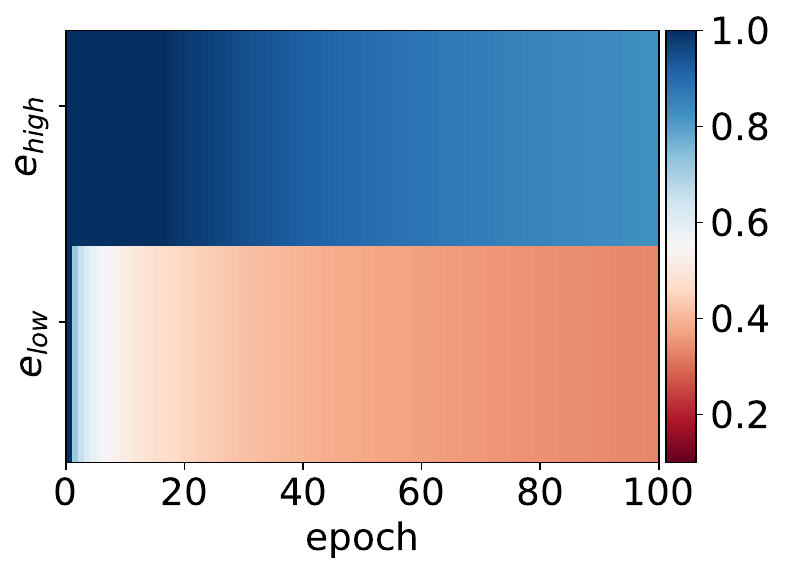} 
\par\end{centering}
}\subfloat[$\delta=3$, CNN]{\begin{centering}
\includegraphics[scale=0.22]{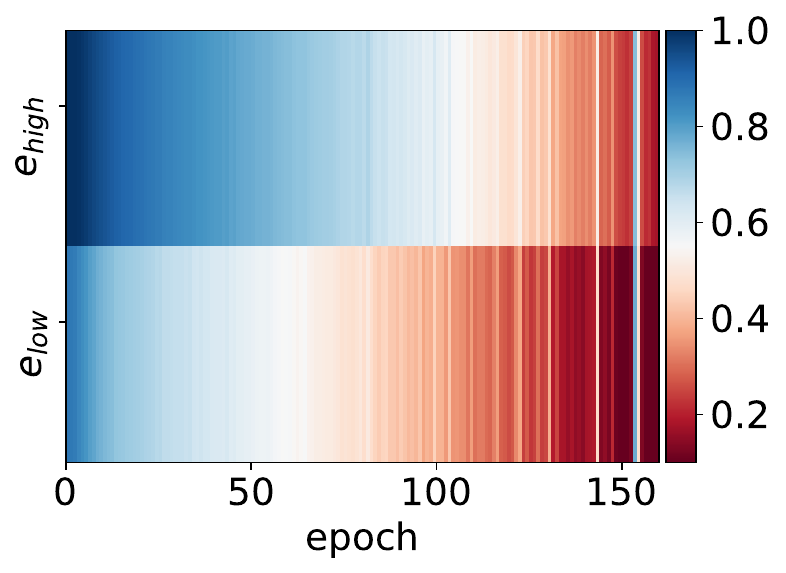} 
\par\end{centering}
}\subfloat[$\delta=7$, VGG]{\begin{centering}
\includegraphics[scale=0.22]{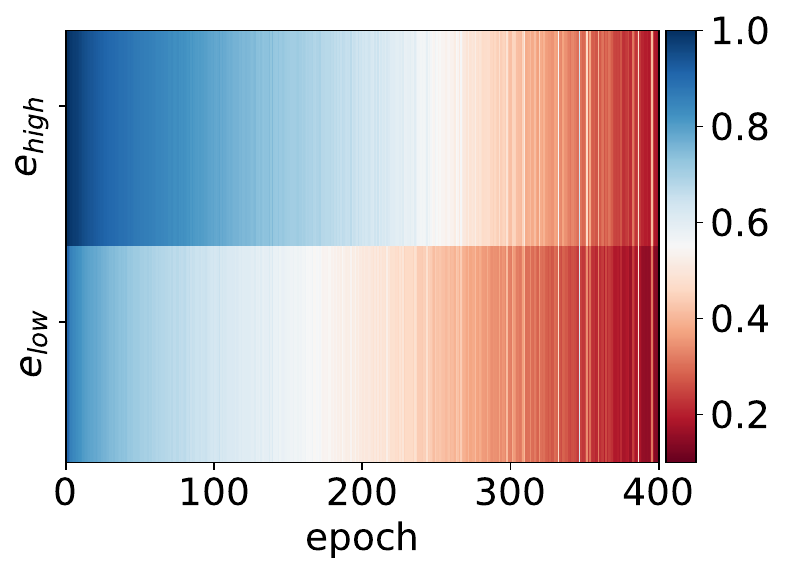} 
\par\end{centering}
}
\par\end{centering}
\begin{centering}
\subfloat[$\delta=7$, DNN]{\begin{centering}
\includegraphics[scale=0.22]{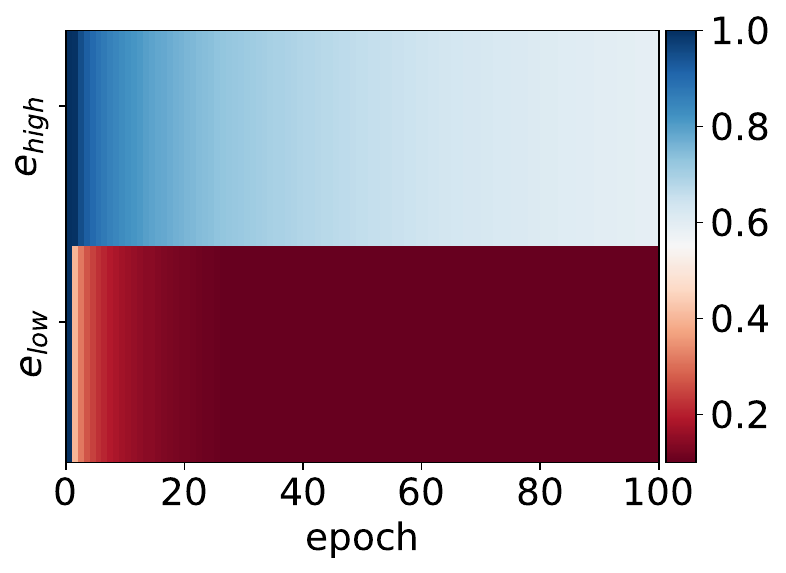} 
\par\end{centering}
}\subfloat[$\delta=7$, CNN]{\begin{centering}
\includegraphics[scale=0.22]{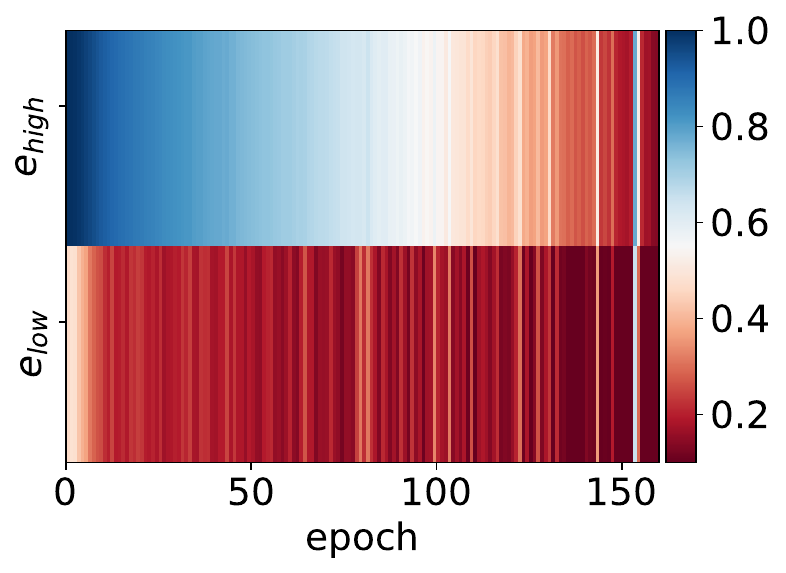} 
\par\end{centering}
}\subfloat[$\delta=10$, VGG]{\begin{centering}
\includegraphics[scale=0.22]{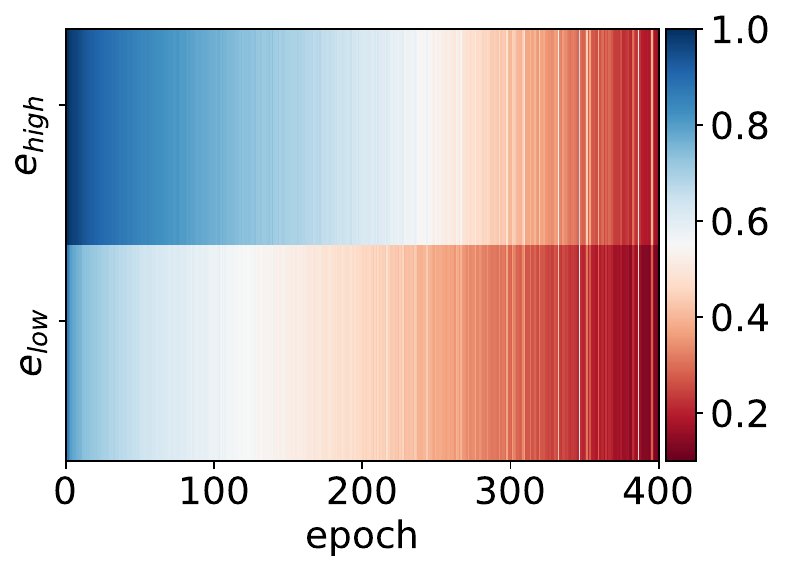} 
\par\end{centering}
}
\par\end{centering}
\caption{F-Principle in real datasets. $e_{\mathrm{low}}$ and $e_{\mathrm{high}}$
indicated by color against training epoch. \label{fig:Noisefitting-Mnist} }
\end{figure}
\par\end{center}

\section{F-Principle in solving differential equation}

Recently, DNN-based approaches have been actively explored for a variety
of scientific computing problems, e.g., solving high-dimensional partial
differential equations \citep{weinan2017deep,khoo2017solving,he2018relu,fan2018multiscale}
and molecular dynamics (MD) simulations \citep{han2017deep}. However,
the behaviors of DNNs applied to these problems are not well-understood.
To facilitate the designs and applications of DNN-based schemes, it
is important to characterize the difference between DNNs and conventional
numerical schemes on various scientific computing problems. In this
section, focusing on solving Poisson's equation, which has broad applications
in mechanical engineering and theoretical physics \citep{evans2010partial},
we highlight a stark difference of a DNN-based solver and the Jacobi
method during the training/iteration, which can be explained by the
F-Principle.

Consider a $1$-d Poisson's equation: 
\begin{align}
    & -\Delta u(x)=g(x),\quad x\in\Omega\triangleq(-1,1),\label{eq:Poisson1-1}\\
    & u(-1)=u(1)=0.\label{eq:Poisson1-2}
\end{align}
We consider the example with $g(x)=\sin(x)+4\sin(4x)-8\sin(8x)+16\sin(24x)$
which has analytic solution $u_{\mathrm{ref}}(x)=g_{0}(x)+c_{1}x+c_{0}$,
where $g_{0}=\sin(x)+\sin(4x)/4-\sin(8x)/8+\sin(24x)/36$, $c_{1}=(g_{0}(-1)-g_{0}(1))/2$
and $c_{0}=-(g_{0}(-1)+g_{0}(1))/2$. $1001$ training samples $\{x_{i}\}_{i=0}^{n}$
are evenly spaced with grid size $\delta x$ in $[0,1]$. Here, we
use the DNN output, $h(x;\theta)$, to fit $u_{\mathrm{ref}}(x)$ (Fig.~\ref{fig:Poisson}(a)).
A DNN-based scheme is proposed by considering the following empirical
loss function \citep{weinan2018deep}, 
\begin{equation}
I_{\mathrm{emp}}=\sum_{i=1}^{n-1}\left(\frac{1}{2}|\nabla_{x}h(x_{i})|^{2}-g(x_{i})h(x_{i})\right)\delta x+\beta\left(h(x_{0})^{2}+h(x_{n})^{2}\right).\label{eq:Energy-1}
\end{equation}
The second term in $I_{\mathrm{emp}}(h)$ is a penalty, with constant $\beta$, arising
from the Dirichlet boundary condition \eqref{eq:Poisson1-2}. After
training, the DNN output well matches the analytical solution $u_{\mathrm{ref}}$.
Focusing on the convergence of three peaks (inset of Fig.~\ref{fig:Poisson}(a))
in the Fourier transform of $u_{\mathrm{ref}}$, as shown in Fig.~\ref{fig:Poisson}(b),
low frequencies converge faster than high frequencies as predicted
by the F-Principle. For comparison, we also use the Jacobi method
to solve problem (\ref{eq:Poisson1-1}). High frequencies converge
faster in the Jacobi method (Details can be found in Appendix \ref{sec:AppPoisson's-equations}),
as shown in Fig.~\ref{fig:Poisson}(c).

As a demonstration, we further propose that DNN can be combined with
conventional numerical schemes to accelerate the convergence of low
frequencies for computational problems. First, we solve the Poisson's
equation in Eq.~(\ref{eq:Poisson1-1}) by DNN with $M$ optimization
steps (or epochs), which needs to be chosen carefully, to get a good
initial guess in the sense that this solution has already learned
the low frequencies (large eigenvalues) part. Then, we use the Jacobi
method with the new initial data for the further iterations. We use
$|h-u_{\mathrm{ref}}|_{\infty}\triangleq\max_{x\in\Omega}|h(x)-u_{\mathrm{ref}}(x)|$
to quantify the learning result. As shown by green stars in Fig.~\ref{fig:Poisson}(d),
$|h-u_{\mathrm{ref}}|_{\infty}$ fluctuates after some running time
using DNN only. Dashed lines indicate the evolution of the Jacobi
method with initial data set to the DNN output at the corresponding
steps. If $M$ is too small (stop too early) (left dashed line), \emph{which
is equivalent to only using Jacobi}, it would take long time to converge
to a small error, because low frequencies converges slowly, yet. If
$M$ is too big (stop too late) (right dashed line), \emph{which is
equivalent to using DNN only}, much time would be wasted for the slow
convergence of high frequencies. A proper choice of $M$ is indicated
by the initial point of orange dashed line, in which low frequencies
are quickly captured by the DNN, followed by fast convergence in high
frequencies of the Jacobi method.

This example illustrates a cautionary tale that, although DNNs has
clear advantage, using DNNs alone may not be the best option because
of its limitation of slow convergence at high frequencies. Taking
advantage of both DNNs and conventional methods to design faster schemes
could be a promising direction in scientific computing problems. 
\begin{center}
\begin{figure*}
\begin{centering}
\includegraphics[scale=0.75]{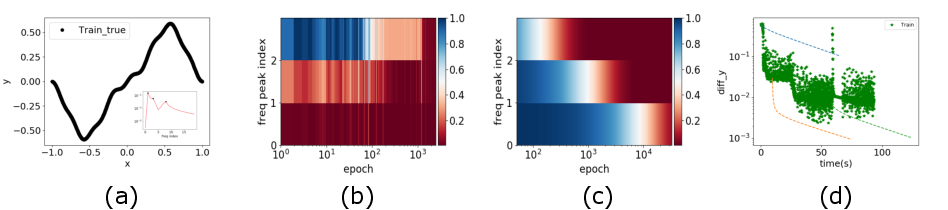} 
\par\end{centering}
\caption{Poisson's equation. (a) $u_{\mathrm{ref}}(x)$. Inset: $|\hat{u}_{\mathrm{ref}}(k)|$
as a function of frequency. Frequencies peaks are marked with black
dots. (b,c) $\Delta_{F}(k)$ computed on the inputs of training data
at different epochs for the selected frequencies for DNN (b) and Jacobi
(c). (d) $|h-u_{\mathrm{ref}}|_{\infty}$ at different running time.
Green stars indicate $|h-u_{\mathrm{ref}}|_{\infty}$ using DNN alone.
The dashed lines indicate $|h-u_{\mathrm{ref}}|_{\infty}$ for the Jacobi
method with different colors indicating initialization by different
timing of DNN training.\label{fig:Poisson} }
\end{figure*}
\par\end{center}

\section{A preliminary theoretical understanding \emph{\label{sec:Theoretical-framework} }}

A subsequent theoretical work \citep{luo2019theory} provides a rigorous
mathematical study of the F-Principle at different frequencies for
general DNNs (e.g., multiple hidden layers, different activation functions,
high-dimensional inputs). The key insight is that the regularity of
DNN converts into the decay rate of a loss function in the frequency
domain. For an intuitive understanding of this key insight, we present
a simplified theory, which connects the smoothness of the activation function
with different gradient priorities in frequency domain. Although this
naive theory is in an ideal setting, it is much easier for understanding.

The activation function we consider is $\sigma(x)=\tanh(x)$, which
is smooth in spatial domain and its derivative decays exponentially with respect to frequency
in the Fourier domain. For a DNN of one hidden layer with $m$ nodes,
1-d input $x$ and 1-d output: $h(x)=\sum_{j=1}^{m}a_{j}\sigma(w_{j}x+b_{j}),\quad a_{j},w_{j},b_{j}\in{\rm \mathbb{R}}.$
We also use the notation $\theta=\{\theta_{lj}\}$ with $\theta_{1j}=a_{j}$,
$\theta_{2j}=w_{j}$, and $\theta_{3j}=b_{j}$, $j=1,\cdots,m$. The
loss at frequency $k$ is $L(k)=\frac{1}{2}\left|\hat{h}(k)-\hat{f}(k)\right|^{2}$,
$\hat{\cdot}$ is the Fourier transform, $f$ is the target function.
The total loss function is defined as: $L=\int_{-\infty}^{+\infty}L(k)\diff{k}$.
Note that according to the Parseval's theorem, this loss function
in the Fourier domain is equal to the commonly used MSE loss. We have
the following theorems (The proofs are at Appendix \ref{sec:Proof-of-theorem1}.).
Define $W=(w_{1},w_{2},\cdots,w_{m})^{T}\in\mathbb{R}^{m}.$ 
\begin{thm}
    \label{thm:Priority} Considering a DNN of one hidden layer with activation
    function $\sigma(x)=\tanh(x)$, for any frequencies $k_{1}$ and $k_{2}$
    such that $|\hat{f}(k_{1})|>0$, $|\hat{f}(k_{2})|>0$, and $|k_{2}|>|k_{1}|>0$,
    there exist positive constants $c$ and $C$ such that for sufficiently
    small $\delta$, we have 
    \begin{align*}
        \frac{\mu\left(\left\{ W:\left|\frac{\partial L(k_{1})}{\partial\theta_{lj}}\right|>\left|\frac{\partial L(k_{2})}{\partial\theta_{lj}}\right|\quad\text{for all}\quad l,j\right\} \cap B_{\delta}\right)}{\mu(B_{\delta})}\geq1-C\exp(-c/\delta),
    \end{align*}
    where $B_{\delta}\subset\mathbb{R}^{m}$ is a ball with radius $\delta$
    centered at the origin and $\mu(\cdot)$ is the Lebesgue measure. 
\end{thm}

Theorem~\ref{thm:Priority} indicates that for any two non-converged
frequencies, with small weights, the lower-frequency gradient exponentially
dominates over the higher-frequency ones. Due to the Parseval's theorem, the MSE loss in the spatial domain is equivalent to the L2 loss in the Fourier domain. To intuitively understand the higher decay rate of a lower-frequency loss function, we consider the training in the Fourier domain with loss function of only two non-zero frequencies.
\begin{thm}
    \label{thm:Priority-2} Considering a DNN of one hidden layer with
    activation function $\sigma(x)=\tanh(x)$. Suppose the target function
    has only two non-zero frequencies $k_{1}$ and $k_{2}$, that is,
    $|\hat{f}(k_{1})|>0$, $|\hat{f}(k_{2})|>0$, $|k_{2}|>|k_{1}|>0$,
    and $|\hat{f}(k)|=0$ for $k\neq k_{1},k_{2}$.  Consider the  loss function of $L=L(k_{1})+L(k_{2})$ with gradient descent training. Denote 
    \begin{equation*}
        \mathcal{S}=\left\{ \frac{\partial L(k_{1})}{\partial t}\leq0,\frac{\partial L(k_{1})}{\partial t}\leq\frac{\partial L(k_{2})}{\partial t}\right\},
    \end{equation*}
    that is, $L(k_{1})$ decreases faster than $L(k_{2})$. There exist
    positive constants $c$ and $C$ such that for sufficiently small
    $\delta$, we have 
    \begin{equation*}
    \frac{\mu\left(\left\{ W:\mathcal{S}\quad{\rm holds}\right\} \cap B_{\delta}\right)}{\mu(B_{\delta})}\geq1-C\exp(-c/\delta),
    \end{equation*}
    where $B_{\delta}\subset\mathbb{R}^{m}$ is a ball with radius $\delta$
    centered at the origin and $\mu(\cdot)$ is the Lebesgue measure. 
\end{thm}

\section{Discussions \label{sec:Discussions}}

Next, we discuss DNN's generalization ability from the view point
of Fourier analysis.

DNNs often generalize well in real problems \citep{zhang2016understanding}
but badly in fitting the parity function \citep{shalev2017failures,nye2018efficient}.
Understanding the differences between above two types of problems,
i.e., good and bad generalization performance of DNN, is critical.
Next, we show a qualitative difference between these two types of
problems through \emph{Fourier analysis} and use the \emph{F-Principle}
to provide insight into how different characteristics in Fourier domain
result in different generalization performances of DNNs.

For MNIST/CIFAR10, we examined $\hat{y}_{\mathrm{total},\vec{k}}=\frac{1}{n_{\mathrm{total}}}\sum_{i=0}^{n_{\mathrm{total}}-1}y_{i}\exp\left(-\mathrm{i}2\pi\vec{k}\cdot\vec{x}_{i}\right)$,
where $\{(\vec{x}_{i},y_{i})\}_{i=0}^{n_{\mathrm{total}}-1}$ consists
of both the training and test datasets with certain selected output
component, at different directions of $\vec{k}$ in the Fourier
space. We find that $\hat{y}_{\mathrm{total},\vec{k}}$ concentrates
on the low frequencies along those examined directions. For illustration,
$\hat{y}_{\mathrm{total},\vec{k}}$'s along the first principle component
are shown by green lines in Fig.~\ref{fig:parity}(a, b) for MNIST/CIFAR10,
respectively. When only the training dataset is used, $\hat{y}_{\mathrm{train},\vec{k}}$
well overlaps with $\hat{y}_{\mathrm{total},\vec{k}}$ at the dominant
low frequencies.

For the parity function $f(\vec{x})=\prod_{j=1}^d x_{j}$ defined on
$\Omega=\{-1,1\}^{d}$, its Fourier transform is $\hat{f}(\vec{k})=\frac{1}{2^{d}}\sum_{x\in\Omega}\prod_{j=1}^d x_{j}\E^{-\mathrm{i}2\pi\vec{k}\cdot\vec{x}}=(-\I)^{d}\prod_{j=1}^d\sin2\pi k_{j}$.
Clearly, for $\vec{k}\in[-\frac{1}{4},\frac{1}{4}]^{d}$,
the power of the parity function concentrates at $\vec{k}\in\{-\frac{1}{4},\frac{1}{4}\}^{d}$
and vanishes as $\vec{k}\to\vec{0}$, as illustrated in Fig.~\ref{fig:parity}(c)
for the direction of $\vec{1}_{d}$. Given a randomly sampled
training dataset $S\subset\Omega$ with $s$ points, the nonuniform
Fourier transform on $S$ is computed as $\hat{f}_{S}(\vec{k})=\frac{1}{s}\sum_{x\in S}\prod_{j=1}^d x_{j}\E^{-\I 2\pi\vec{k}\cdot\vec{x}}$.
As shown in Fig.~\ref{fig:parity}(c), $\hat{f}(\vec{k})$ and
$\hat{f}_{S}(\vec{k})$ significantly differ at low frequencies.

By experiments, the generalization ability of DNNs can be well reflected
by the Fourier analysis. For the MNIST/CIFAR10, we observed the Fourier
transform of the output of a well-trained DNN on $\{\vec{x}_{i}\}_{i=0}^{n_{\mathrm{total}}-1}$
faithfully recovers the dominant low frequencies, as illustrated in
Fig.~\ref{fig:parity}(a) and \ref{fig:parity}(b), respectively, 
indicating a good generalization performance as observed in experiments.
However, for the parity function, we observed that the Fourier transform
of the output of a well-trained DNN on $\{\vec{x}_{i}\}_{i\in S}$
significantly deviates from $\hat{f}(\vec{k})$ at almost all frequencies,
as illustrated in Fig.~\ref{fig:parity}(c), indicating a bad generalization
performance as observed in experiments.

The F-Principle implicates that among all the functions that can
fit the training data, a DNN is implicitly biased during the training
towards a function with more power at low frequencies. If the target
function has significant high-frequency components, insufficient training
samples will lead to artificial low frequencies in training dataset,
such as the parity function as shown in Fig.~\ref{fig:parity}(c),
which is the well-known \emph{aliasing} effect. Based on the F-Principle,
as demonstrated in Fig.~\ref{fig:parity}(c), these artificial low
frequency components will be first captured to explain the training
samples, whereas the high frequency components will be compromised
by DNN. For MNIST/CIFAR10, since the power of high frequencies is
much smaller than that of low frequencies, artificial low frequencies
caused by aliasing can be neglected. To conclude, the distribution
of power in Fourier domain of above two types of problems exhibits
significant differences, which result in different generalization
performances of DNNs according to the F-Principle. 
\begin{center}
\begin{figure*}
\begin{centering}
\includegraphics[scale=0.65]{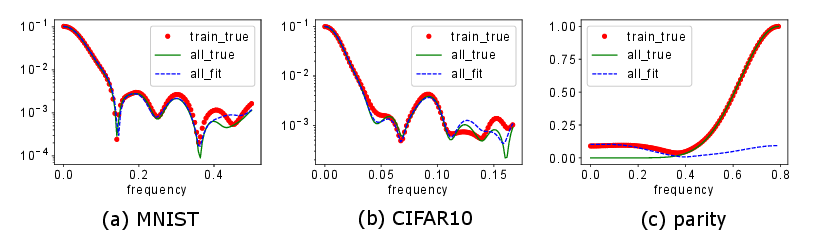} 
\par\end{centering}
\caption{Fourier analysis for different generalization ability. The plot is
the amplitude of the Fourier coefficient against frequency $k$. The
red dots are for the training dataset, the green line is for the whole
dataset, and the blue dashed line is for an output of well-trained
DNN on the input of the whole dataset. For (c), $d=10$.
The training data is $200$ randomly selected points. \label{fig:parity}}
\end{figure*}
\par\end{center}

\section{Related work}

As widely observed in experiments, DNNs with gradient-based training
show different generalization abilities for different problems. On
the one hand, there are different approaches attempting to explain
why the training process often leads to a DNN of good generalization
ability even when the number of parameters is much larger than the
number of training data \citep{zhang2016understanding}. For example,
generalization error is related to various complexity measures \citep{bartlett1999almost,bartlett2002rademacher,bartlett2017spectrally,bartlett2017nearly,neyshabur2017exploring,golowich2017size,dziugaite2017computing,neyshabur2018towards,ma2018priori},
local properties (sharpness/flatness) of loss functions at minima
\citep{hochreiter1995simplifying,keskar2016large,dinh2017sharp,wu2017towards},
stability of optimization algorithms \citep{bousquet2002stability,xu2012robustness,hardt2015train},
and implicit bias of the training process \citep{neyshabur2014search,poggio2018theory,soudry2018implicit,arpit2017closer,xu_training_2018}.
On the other hand, several works focus on the failure of DNNs \citep{shalev2017failures,nye2018efficient},
e.g., fitting the parity function, in which a well-trained DNN possesses
no generalization ability.

In the revised version, \cite{rahaman2018spectral} also examines
the F-Principle in the MNIST dataset. However, they add noise to MNIST,
which contaminates the labels. They only examine not very deep (6-layer)
fully connected ReLU network with MSE loss, while cross-entropy loss
is widely used.

\section*{Acknowledgments}

The authors want to thank Weinan E (Princeton University) and David W. McLaughlin (New York University) for helpful discussions
and thank Yang Yuan (Tsinghua University) and Zhanxing Zhu (Peking University) for critically reading the
manuscript. Part of this work was done when ZX, YZ, YX are postdocs at New York University Abu Dhabi and visiting members at Courant Institute supported by the NYU Abu Dhabi Institute
G1301. The authors declare no conflict of interest.

\bibliographystyle{iclr2020_conference}
\bibliography{DLRef}

\appendix
%dummy comment inserted by tex2lyx to ensure that this paragraph is not empty

\section{Illustration of F-Principle for $1$-d synthetic data \label{sec:1dexp}}

To illustrate the phenomenon of F-Principle, we use $1$-d synthetic
data to show the evolution of relative training error at different
frequencies during the training of DNN. we train a DNN to fit a $1$-d
target function $f(x)=\sin(x)+\sin(3x)+\sin(5x)$ of three frequency
components. On $n=201$ evenly spaced training samples, i.e., $\{x_{i}\}_{i=0}^{n-1}$
in $[-3.14,3.14]$, the discrete Fourier transform (DFT) of $f(x)$
or the DNN output (denoted by $h(x)$) is computed by $\hat{f}_{k}=\frac{1}{n}\sum_{i=0}^{n-1}f(x_{i})\E^{-\I2\pi ik/n}$
and $\hat{h}_{k}=\frac{1}{n}\sum_{i=0}^{n-1}h(x_{i})\E^{-\I2\pi jk/n}$,
where $k$ is the frequency. As shown in Fig.~\ref{fig:onelayer}(a),
the target function has three important frequencies as we design (black
dots at the inset in Fig.~\ref{fig:onelayer}(a)). To examine the
convergence behavior of different frequency components during the
training with MSE, we compute the relative difference of the DNN output
and the target function for the three important frequencies $k$'s
at each recording step, that is, $\Delta_{F}(k)=|\hat{h}_{k}-\hat{f}_{k}|/|\hat{f}_{k}|$,
where $|\cdot|$ denotes the norm of a complex number. As shown in
Fig.~\ref{fig:onelayer}(b), the DNN converges the first frequency
peak very fast, while converging the second frequency peak much slower,
followed by the third frequency peak.

Next, we investigate the F-Principle on real datasets with more general
loss functions other than MSE which was the only loss studied in the
previous works \citep{xu_training_2018,rahaman2018spectral}. All
experimental details can be found in Appendix. \ref{sec:Experimental-settings}. 
\begin{center}
\begin{figure}
\begin{centering}
\includegraphics[scale=0.6]{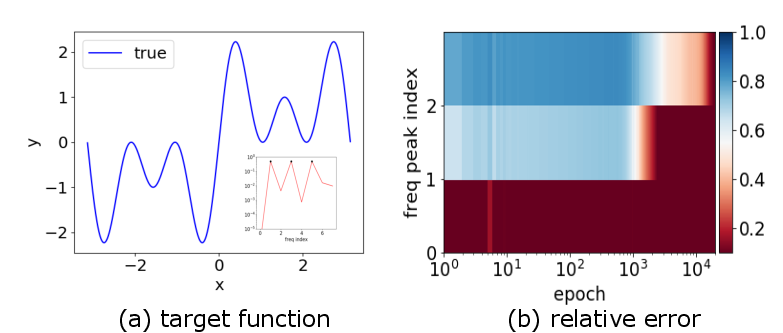} 
\par\end{centering}
\caption{1d input. (a) $f(x)$. Inset : $|\hat{f}(k)|$. (b) $\Delta_{F}(k)$
of three important frequencies (indicated by black dots in the inset
of (a)) against different training epochs.\label{fig:onelayer} }
\end{figure}
\par\end{center}

\section{Experimental settings\label{sec:Experimental-settings}}

In Fig. \ref{fig:onelayer}, the parameters of the DNN is initialized
by a Gaussian distribution with mean $0$ and standard deviation $0.1$.
We use a tanh-DNN with widths $1$-$8000$-$1$ with full batch training. The
learning rate is $0.0002$. The DNN is trained by Adam optimizer \citep{kingma2014adam}
with the MSE loss function.

In Fig. \ref{fig:CMFT}, for MNIST dataset, the training process of
a tanh-DNN with widths $784$-$400$-$200$-$10$ is shown in Fig.~\ref{fig:CMFT}(a)
and \ref{fig:CMFT}(b). For CIFAR10 dataset, results are shown in
Fig.~\ref{fig:CMFT}(c) and \ref{fig:CMFT}(d) of a ReLU-CNN, which
consists of one convolution layer of $3\times3\times64$, a max pooling
of $2\times2$, one convolution layer of $3\times3\times128$, a max
pooling of $2\times2$, followed by a fully-connected DNN with widths
$800$-$400$-$400$-$400$-$10$. For both cases, the output layer
of the network is equipped with a softmax. The network output is a
$10$-d vector. The DNNs are trained with cross entropy loss by Adam
optimizer \citep{kingma2014adam}. (a, b) are for MNIST with a tanh-DNN.
The learning rate is $0.001$ with batch size $10000$. After training,
the training accuracy is $0.951$ and test accuracy is $0.963$. The
amplitude of the Fourier coefficient with respect to the fourth output
component at each frequency is shown in (a), in which the red dots
are computed using the training data. Selected frequencies are marked
by black squares. (b) $\Delta_{F}(k)$ at different training epochs
for the selected frequencies. (c, d) are for CIFAR10 dataset. We use
a ReLU network of a CNN followed by a fully-connected DNN. The learning
rate is $0.003$ with batch size $512$. (c) and (d) are the results
with respect to the ninth output component. After training, the training
accuracy is $0.98$ and test accuracy is $0.72$.

In Fig. \ref{fig:Noisefitting-Mnist}, for MNIST, we use a fully-connected
tanh-DNN with widths $784$-$400$-$200$-$10$ and MSE loss; for
CIFAR10, we use cross-entropy loss and a ReLU-CNN, which consists
of one convolution layer of $3\times3\times32$, a max pooling of
$2\times2$, one convolution layer of $3\times3\times64$, a max pooling
of $2\times2$, followed by a fully-connected DNN with widths $400$-$10$
and the output layer of the network is equipped with a softmax. The
learning rate for MNIST and CIFAR10 is $0.015$ and $0.003$, respectively.
The networks are trained by Adam optimizer \citep{kingma2014adam}
with batch size $10000$. For VGG16, the learning rate is $10^{-5}$. The network is trained by Adam optimizer
\citep{kingma2014adam} with batch size $500$.

In Fig. \ref{fig:Poisson}, the samples are evenly spaced in $[0,1]$
with sample size $1001$. We use a DNN with widths 1-4000-500-400-1
and full batch training by Adam optimizer \citep{kingma2014adam}.
The learning rate is $0.0005$. $\beta$ is $10$. The parameters
of the DNN are initialized following a Gaussian distribution with
mean $0$ and standard deviation $0.02$.

In Fig. \ref{fig:parity}, the settings of (a) and (b) are the same
as the ones in Fig.~\ref{fig:CMFT}. For (c), we use a tanh-DNN with
widths 10-500-100-1, learning rate $0.0005$ under full batch-size
training by Adam optimizer \citep{kingma2014adam}. The parameters
of the DNN are initialized by a Gaussian distribution with mean $0$
and standard deviation $0.05$.

\section{Central difference scheme and Jacobi method\label{sec:AppPoisson's-equations}}

Consider a one-dimensional (1-d) Poisson's equation: 
\begin{equation}
    -\Delta u(x)=g(x),\quad x\in\Omega=(-1,1)\label{eq:Poisson1}
\end{equation}
\begin{equation*}
    u(x)=0,\quad x=-1,1.
\end{equation*}

$[-1,1]$ is uniformly discretized into $n+1$ points with grid size
$h=2/n$. The Poisson's equation in Eq.~(\ref{eq:Poisson1}) can
be solved by the central difference scheme, 
\begin{equation}
    -\Delta u_{i}=-\frac{u_{i+1}-2u_{i}+u_{i-1}}{(\delta x)^{2}}=g(x_{i}),\quad i=1,2,\cdots,n,
\end{equation}
resulting a linear system 
\begin{equation}
    \mat{A}\vec{u}=\vec{g},\label{eq:auf}
\end{equation}
where 
\begin{equation}
    \mat{A}=\left(\begin{array}{cccccc}
    2 & -1 & 0 & 0 & \cdots & 0\\
    -1 & 2 & -1 & 0 & \cdots & 0\\
    0 & -1 & 2 & -1 & \cdots & 0\\
    \vdots & \vdots & \cdots &  &  & \vdots\\
    0 & 0 & \cdots & 0 & -1 & 2
    \end{array}\right)_{(n-1)\times(n-1)},
\end{equation}

\begin{equation}
    \vec{u}=\left(\begin{array}{c}
    u_{1}\\
    u_{2}\\
    \vdots\\
    u_{n-2}\\
    u_{n-1}
    \end{array}\right),\quad
    \vec{g}=(\delta x)^{2}\left(\begin{array}{c}
    g_{1}\\
    g_{2}\\
    \vdots\\
    g_{n-2}\\
    g_{n-1}
    \end{array}\right),\quad x_{i}=2\frac{i}{n}.
\end{equation}
A class of methods to solve this linear system is iterative schemes,
for example, the Jacobi method. Let $\mat{A}=\mat{D}-\mat{L}-\mat{U}$, where $\mat{D}$ is the
diagonal of $\mat{A}$, and $\mat{L}$ and $\mat{U}$ are the strictly lower and upper
triangular parts of $-\mat{A}$, respectively. Then, we obtain 
\begin{equation}
    \vec{u}=\mat{D}^{-1}(\mat{L}+\mat{U})\vec{u}+\mat{D}^{-1}\vec{g}.
\end{equation}
At step $t\in\N$, the Jacobi iteration reads as
\begin{equation}
    \vec{u}^{t+1}=\mat{D}^{-1}(\mat{L}+\mat{U})\vec{u}^{t}+\mat{D}^{-1}\vec{g}.
\end{equation}
We perform the standard error analysis of the above iteration process.
Denote $\vec{u}^{*}$ as the true value obtained by directly performing
inverse of $\mat{A}$ in Eq.~(\ref{eq:auf}). The error at step $t+1$
is $\vec{e}^{t+1}=\vec{u}^{t+1}-\vec{u}^{*}$. Then, $\vec{e}^{t+1}=\mat{R}_{J}\vec{e}^{t}$,
where $\mat{R}_{J}=\mat{D}^{-1}(\mat{L}+\mat{U})$. The converging speed of $\vec{e}^{t}$
is determined by the eigenvalues of $\mat{R}_{J}$, that is, 
\begin{equation}
    \lambda_{k}=\lambda_{k}(\mat{R}_{J})=\cos\frac{k\pi}{n},\quad k=1,2,\cdots,n-1,
\end{equation}
and the corresponding eigenvector $\vec{v}_{k}$'s entry is 
\begin{equation}
    v_{k,i}=\sin\frac{ik\pi}{n},i=1,2,\cdots,n-1.
\end{equation}
So we can write 
\begin{equation}
    \vec{e}^{t}=\sum_{k=1}^{n-1}\alpha_{k}^{t}\vec{v}_{k},
\end{equation}
where $\alpha_{k}^{t}$ can be understood as the magnitude of $\vec{e}^{t}$
in the direction of $\vec{v}_{k}$. Then, 
\begin{equation}
    \vec{e}^{t+1}=\sum_{k=1}^{n-1}\alpha_{k}^{t}\mat{R}_{J}\vec{v}_{k}=\sum_{k=1}^{n-1}\alpha_{k}^{t}\lambda_{k}\vec{v}_{k}.
\end{equation}
\begin{equation*}
    \alpha_{k}^{t+1}=\lambda_{k}\alpha_{k}^{t}.
\end{equation*}
Therefore, the converging rate of $\vec{e}^{t}$ in the direction
of $\vec{v}_{k}$ is controlled by $\lambda_{k}$. Since 
\begin{equation}
    \cos\frac{k\pi}{n}=-\cos\frac{(n-k)\pi}{n},
\end{equation}
the frequencies $k$ and $(n-k)$ are closely related and converge
with the same rate. Consider the frequency $k<n/2$, $\lambda_{k}$
is larger for lower frequency. Therefore, lower frequency converges
slower in the Jacobi method.

\section{Proof of theorems\label{sec:Proof-of-theorem1}}

The activation function we consider is $\sigma(x)=\tanh(x)$. 
\begin{equation*}
    \sigma(x)=\tanh(x)=\frac{\E^{x}-\E^{-x}}{\E^{x}+\E^{-x}},\quad x\in\mathbb{R}.
\end{equation*}
For a DNN of one hidden layer with $m$ nodes, 1-d input $x$ and
1-d output: 
\begin{equation}
    h(x)=\sum_{j=1}^{m}a_{j}\sigma(w_{j}x+b_{j}),\quad a_{j},w_{j},b_{j}\in{\rm \mathbb{R}},\label{eq: DNNmath}
\end{equation}
where $w_{j}$, $a_{j}$, and $b_{j}$ are called \emph{parameters},
in particular, $w_{j}$ and $a_{j}$ are called \emph{weights}, and
$b_{j}$ is also known as a \emph{bias}. In the sequel, we will also
use the notation $\theta=\{\theta_{lj}\}$ with $\theta_{1j}=a_{j}$,
$\theta_{2j}=w_{j}$, and $\theta_{lj}=b_{j}$, $j=1,\cdots,m$. Note
that $\hat{\sigma}(k)=-\frac{\I\pi}{\sinh(\pi k/2)}$ where
the Fourier transformation and its inverse transformation are defined
as follows: 
\begin{equation*}
    \hat{f}(k)=\int_{-\infty}^{+\infty}f(x)\E^{-\I kx}\diff{x},\quad f(x)=\frac{1}{2\pi}\int_{-\infty}^{+\infty}\hat{f}(k)\E^{\I kx}\diff{k}.
\end{equation*}
The Fourier transform of $\sigma(w_{j}x+b_{j})$ with $w_{j},b_{j}\in{\rm \mathbb{R}}$,
$j=1,\cdots,m$ reads as 
\begin{equation}
    \widehat{\sigma(w_{j}\cdot+b_{j})}(k)=\frac{2\pi\I}{|w_{j}|}\exp\Big(\frac{\I b_{j}k}{w_{j}}\Big)\frac{1}{\exp(-\frac{\pi k}{2w_{j}})-\exp(\frac{\pi k}{2w_{j}})}.\label{eq:FSigOri}
\end{equation}
Thus 
\begin{equation}
    \hat{h}(k)=\sum_{j=1}^{m}\frac{2\pi a_{j}\I}{|w_{j}|}\exp\Big(\frac{\I b_{j}k}{w_{j}}\Big)\frac{1}{\exp(-\frac{\pi k}{2w_{j}})-\exp(\frac{\pi k}{2w_{j}})}.\label{eq:FTW}
\end{equation}
We define the amplitude deviation between DNN output and the \emph{target
function} $f(x)$ at frequency $k$ as 
\begin{equation*}
    D(k)\triangleq\hat{h}(k)-\hat{f}(k).
\end{equation*}
Write $D(k)$ as $D(k)=A(k)\E^{\I\phi(k)}$, where $A(k)\in[0,+\infty)$
and $\phi(k)\in\mathbb{R}$ are the amplitude and phase of $D(k)$,
respectively. The loss at frequency $k$ is $L(k)=\frac{1}{2}\left|D(k)\right|^{2}$,
where $|\cdot|$ denotes the norm of a complex number. The total loss
function is defined as: $L=\int_{-\infty}^{+\infty}L(k)\diff{k}$.
Note that according to the Parseval's theorem, this loss function
in the Fourier domain is equal to the commonly used loss of mean squared
error, that is, $L=\int_{-\infty}^{+\infty}\frac{1}{2}(h(x)-f(x))^{2}\diff{x}$.
For readers' reference, we list the partial derivatives of $L(k)$
with respect to parameters 
\begin{align}
    \frac{\partial L(k)}{\partial a_{j}} & =\frac{2\pi}{w_{j}}\sin\Big(\frac{b_{j}k}{w_{j}}-\phi(k)\Big)E_{0},\label{eq:DLalpha}\\
    \frac{\partial L(k)}{\partial w_{j}} & =\Bigg[\sin\Big(\frac{b_{j}k}{w_{j}}-\phi(k)\Big)\left(\frac{\pi^{2}a_{j}k}{w_{j}^{3}}E_{1}-\frac{2\pi a_{j}}{w_{j}^{2}}\right)\nonumber \\
     & \quad{}-\frac{2\pi a_{j}b_{j}k}{w_{j}^{3}}\cos\Big(\frac{b_{j}k}{w_{j}}-\phi(k)\Big)\Bigg]E_{0},\label{eq:paritialLA}\\
    \frac{\partial L(k)}{\partial b_{j}} & =\frac{2\pi a_{j}b_{j}k}{w_{j}^{2}}\cos\Big(\frac{b_{j}k}{w_{j}}-\phi(k)\Big)E_{0},\label{eq:paritialLB}
\end{align}
where 
\begin{equation*}
    E_{0}=\frac{\mathrm{sgn}(w_{j})A(k)}{\exp(\frac{\pi k}{2w_{j}})-\exp(-\frac{\pi k}{2w_{j}})},
\end{equation*}
\begin{equation*}
    E_{1}=\frac{\exp(\frac{\pi k}{2w_{j}})+\exp(-\frac{\pi k}{2w_{j}})}{\exp(\frac{\pi k}{2w_{j}})-\exp(-\frac{\pi k}{2w_{j}})}.
\end{equation*}
The descent increment at any direction, say, with respect to parameter
$\theta_{lj}$, is 
\begin{equation}
    \frac{\partial L}{\partial\theta_{lj}}=\int_{-\infty}^{+\infty}\frac{\partial L(k)}{\partial\theta_{lj}}\diff{k}.\label{eq:GDfreq}
\end{equation}
The absolute contribution from frequency $k$ to this total amount
at $\theta_{lj}$ is 
\begin{equation}
    \left|\frac{\partial L(k)}{\partial\theta_{lj}}\right|\approx A(k)\exp\left(-|\pi k/2w_{j}|\right)F_{lj}(\theta_{j},k),\label{eq:DL2}
\end{equation}
where $\theta_{j}\triangleq\{w_{j},b_{j},a_{j}\}$, $\theta_{lj}\in\theta_{j}$,
$F_{lj}(\theta_{j},k)$ is a function with respect to $\theta_{j}$
and $k$, which can be found in one of Eqs. (\ref{eq:DLalpha}, \ref{eq:paritialLA},
\ref{eq:paritialLB}).

When the component at frequency $k$ where $\hat{h}(k)$ is not close
enough to $\hat{f}(k)$, $\exp\left(-|\pi k/2w_{j}|\right)$ would
dominate $G_{lj}(\theta_{j},k)$ for a small $w_{j}$. Through the
above framework of analysis, we have the following theorem. Define
\begin{equation}
    W=(w_{1},w_{2},\cdots,w_{m})^{T}\in\mathbb{R}^{m}.
\end{equation}

\begin{thm*}\label{thm:Priority-1}Consider a one hidden layer DNN with activation
    function $\sigma(x)=\tanh{x}$. For any frequencies $k_{1}$ and $k_{2}$
    such that $|\hat{f}(k_{1})|>0$, $|\hat{f}(k_{2})|>0$, and $|k_{2}|>|k_{1}|>0$,
    there exist positive constants $c$ and $C$ such that for sufficiently
    small $\delta$, we have 
    \begin{align}
    \frac{\mu\left(\left\{ W:\left|\frac{\partial L(k_{1})}{\partial\theta_{lj}}\right|>\left|\frac{\partial L(k_{2})}{\partial\theta_{lj}}\right|\quad\text{for all}\quad l,j\right\} \cap B_{\delta}\right)}{\mu(B_{\delta})}\nonumber \\
    \geq1-C\exp(-c/\delta),\label{eq:thm1proof}
    \end{align}
    where $B_{\delta}\subset\mathbb{R}^{m}$ is a ball with radius $\delta$
    centered at the origin and $\mu(\cdot)$ is the Lebesgue measure. 
\end{thm*}
We remark that $c$ and $C$ depend on $k_{1}$, $k_{2}$, $|\hat{f}(k_{1})|$,
$|\hat{f}(k_{2})|$, $\sup{|a_{i}|}$, $\sup{|b_{i}|}$, and $m$. 
\begin{proof}
    To prove the statement, it is sufficient to show that $\mu(S_{lj,\delta})/\mu(B_{\delta})\leq C\exp(-c/\delta)$
    for each $l,j$, where 
    \begin{equation}
        S_{lj,\delta}:=\left\{ W\in B_{\delta}:\left|\frac{\partial L(k_{1})}{\partial\theta_{lj}}\right|\leq\left|\frac{\partial L(k_{2})}{\partial\theta_{lj}}\right|\right\} .
    \end{equation}
    We prove this for $S_{1j,\delta}$, that is, $\theta_{lj}=a_{j}$.
    The proofs for $\theta_{lj}=w_{j}$ and $b_{j}$ are similar. Without
    loss of generality, we assume that $k_{1},k_{2}>0$, $b_{j}>0$, and
    $w_{j}\neq0$, $j=1,\cdots,m$. According to Eq.~(\ref{eq:DLalpha}),
    the inequality $|\frac{\partial L(k_{1})}{\partial a_{j}}|\leq|\frac{\partial L(k_{2})}{\partial a_{j}}|$
    is equivalent to 
    \begin{equation}
        \frac{A(k_{2})}{A(k_{1})}\Bigg|\frac{\exp(\frac{\pi k_{1}}{2w_{j}})-\exp(-\frac{\pi k_{1}}{2w_{j}})}{\exp(\frac{\pi k_{2}}{2w_{j}})-\exp(-\frac{\pi k_{2}}{2w_{j}})}\Bigg|\cdot\Big|\sin\Big(\frac{b_{j}k_{2}}{w_{j}}-\phi(k_{2})\Big)\Big|
        \geq\Big|\sin\Big(\frac{b_{j}k_{1}}{w_{j}}-\phi(k_{1})\Big)\Big|\label{eq..comparison}
    \end{equation}
    Note that $|\hat{h}(k)|\leq C\sum_{j=1}^{m}\frac{|a_{j}|}{|w_{j}|}\exp(-\frac{\pi k}{2|w_{j}|})$
    for $k>0$. Thus 
    \begin{equation}
        \lim_{W\rightarrow0}\hat{h}(k)=0\quad\text{and}\quad\lim_{W\rightarrow0}D(k)=-\hat{f}(k).
    \end{equation}
    Therefore, 
    \begin{equation}
        \lim_{W\rightarrow0}A(k)=|\hat{f}(k)|\quad\text{and}\quad\lim_{W\rightarrow0}\phi(k)=\pi+\arg(\hat{f}(k)).\label{eq..theta.k}
    \end{equation}
    For $W\in B_{\delta}$ with sufficiently small $\delta$, $A(k_{1})>\frac{1}{2}|\hat{f}(k_{1})|>0$
    and $A(k_{2})<2|\hat{f}(k_{2})|$. Also note that $|\sin(\frac{b_{j}k_{2}}{w_{j}}-\phi(k_{2}))|\leq1$
    and that for sufficiently small $\delta$, 
    \begin{equation}
        \Bigg|\frac{\exp(\frac{\pi k_{1}}{2w_{j}})-\exp(-\frac{\pi k_{1}}{2w_{j}})}{\exp(\frac{\pi k_{2}}{2w_{j}})-\exp(-\frac{\pi k_{2}}{2w_{j}})}\Bigg|\leq2\exp\Big(\frac{-\pi(k_{2}-k_{1})}{2|w_{j}|}\Big).
    \end{equation}
    Thus, inequality (\ref{eq..comparison}) implies that 
    \begin{equation}
        \Big|\sin\Big(\frac{b_{j}k_{1}}{w_{j}}-\phi(k_{1})\Big)\Big|\leq\frac{8|\hat{f}(k_{2})|}{|\hat{f}(k_{1})|}\exp\Big(-\frac{\pi(k_{2}-k_{1})}{2|w_{j}|}\Big).\label{eq:expineq}
    \end{equation}
    Noticing that $\frac{2}{\pi}|x|\leq|\sin x|$ ($|x|\leq\frac{\pi}{2}$)
    and Eq.~(\ref{eq..theta.k}), we have for $W\in S_{lj,\delta}$,
    $\text{for some}\quad q\in\mathbb{Z}$, 
    \begin{equation}
        \Big|\frac{b_{i}k_{1}}{w_{i}}-\arg(\hat{f}(k_{1}))-q\pi\Big|\leq\frac{8\pi|\hat{f}(k_{2})|}{|\hat{f}(k_{1})|}\exp\Big(-\frac{\pi(k_{2}-k_{1})}{2\delta}\Big)
    \end{equation}
    that is, 
    \begin{equation}
        -c_{1}\exp(-c_{2}/\delta)+q\pi+\arg(\hat{f}(k_{1}))\leq\frac{b_{i}k_{1}}{w_{i}}
        \leq c_{1}\exp(-c_{2}/\delta)+q\pi+\arg(\hat{f}(k_{1})),
    \end{equation}
    where $c_{1}=\frac{8\pi|\hat{f}(k_{2})|}{|\hat{f}(k_{1})|}$ and $c_{2}=\pi(k_{2}-k_{1})$.
    Define $I:=I^{+}\cup I^{-}$ where 
    \begin{equation}
        I^{+}:=\{w_{j}>0:W\in S_{1j,\delta}\},\quad I^{-}:=\{w_{j}<0:W\in S_{1j,\delta}\}.
    \end{equation}
    For $w_{j}>0$, we have for some $q\in\mathbb{Z}$, 
    \begin{equation}
        0<\frac{b_{j}k_{1}}{c_{1}\exp(-c_{2}/\delta)+q\pi+\arg(\hat{f}(k_{1}))}\leq w_{j}
        \leq\frac{b_{j}k_{1}}{-c_{1}\exp(-c_{2}/\delta)+q\pi+\arg(\hat{f}(k_{1}))}.\label{eq..wSmallMeasure}
    \end{equation}
    Since $W\in B_{\delta}$ and $c_{1}\exp(-c_{2}/\delta)+\arg(\hat{f}(k_{1}))\leq2\pi$,
    we have $\frac{b_{j}k_{1}}{2\pi+q\pi}\leq w_{j}\leq\delta$. Then
    Eq.~(\ref{eq..wSmallMeasure}) only holds for some large $q$, more
    precisely, $q\geq q_{0}:=\frac{b_{j}k}{\pi\delta}-2$. Thus we obtain
    the estimate for the (one-dimensional) Lebesgue measure of $I^{+}$
    \begin{align}
        \mu(I^{+}) 
        & \leq \sum_{q=q_{0}}^{\infty}\Bigg|\frac{b_{j}k_{1}}{-c_{1}\exp(-c_{2}/\delta)+q\pi+\arg(\hat{f}(k_{1}))}-\frac{b_{j}k_{1}}{c_{1}\exp(-c_{2}/\delta)+q\pi+\arg(\hat{f}(k_{1}))}\Bigg|\nonumber \\
        & \leq 2|b_{j}|k_{1}c_{1}\exp(-c_{2}/\delta)\cdot\sum_{q=q_{0}}^{\infty}\frac{1}{(q\pi+\arg(\hat{f}(k_{1})))^{2}-(c_{1}\exp(-c_{2}/\delta))^{2}}\nonumber \\
        & \leq C\exp(-c/\delta).
    \end{align}
    The similar estimate holds for $\mu(I^{-})$, and hence $\mu(I)\leq C\exp(-c/\delta)$.
    For $W\in B_{\delta}$, the $(m-1)$ dimensional vector $(w_{1},\cdots,w_{j-1},w_{j+1},\cdots,w_{m})^{T}$
    is in a ball with radius $\delta$ in $\mathbb{R}^{m-1}$. Therefore,
    we final arrive at the desired estimate 
    \begin{equation}
        \frac{\mu(S_{1j,\delta})}{\mu(B_{\delta})}\leq\frac{\mu(I)\omega_{m-1}\delta^{m-1}}{\omega_{m}\delta^{m}}\leq C\exp(-c/\delta),
    \end{equation}
    where $\omega_{m}$ is the volume of a unit ball in $\mathbb{R}^{m}$. 
\end{proof}
\begin{thm*}
    \label{thm:Priority-2-1} Considering a DNN of one hidden layer with
    activation function $\sigma(x)=\tanh(x)$. Suppose the target function
    has only two non-zero frequencies $k_{1}$ and $k_{2}$, that is,
    $|\hat{f}(k_{1})|>0$, $|\hat{f}(k_{2})|>0$, and $|k_{2}|>|k_{1}|>0$,
    and $|\hat{f}(k)|=0$ for $k\neq k_{1},k_{2}$. Consider the  loss function of $L=L(k_{1})+L(k_{2})$ with gradient descent training. Denote 
    \begin{equation*}
        \mathcal{S}=\left\{ \frac{\partial L(k_{1})}{\partial t}\leq0,\frac{\partial L(k_{1})}{\partial t}\leq\frac{\partial L(k_{2})}{\partial t}\right\} ,
    \end{equation*}
    that is, $L(k_{1})$ decreases faster than $L(k_{2})$. There exist
    positive constants $c$ and $C$ such that for sufficiently small
    $\delta$, we have 
    \begin{equation*}
        \frac{\mu\left(\left\{ W:\mathcal{S}\quad{\rm holds}\right\} \cap B_{\delta}\right)}{\mu(B_{\delta})}\geq1-C\exp(-c/\delta),
    \end{equation*}
    where $B_{\delta}\subset\mathbb{R}^{m}$ is a ball with radius $\delta$
    centered at the origin and $\mu(\cdot)$ is the Lebesgue measure. 
\end{thm*}
\begin{proof}
    By gradient descent algorithm, we obtain
    \begin{align*}
        \frac{\partial L(k_{1})}{\partial t} & =\sum_{l,j}\frac{\partial L(k_{1})}{\partial\theta_{lj}}\frac{\partial\theta_{lj}}{\partial t}\\
        & =-\sum_{l,j}\frac{\partial L(k_{1})}{\partial\theta_{lj}}\frac{\partial(L(k_{1})+L(k_{2}))}{\partial\theta_{lj}}\\
        & =-\sum_{l,j}\left(\frac{\partial L(k_{1})}{\partial\theta_{lj}}\right)^{2}-\sum_{l,j}\frac{\partial L(k_{1})}{\partial\theta_{lj}}\frac{\partial L(k_{2})}{\partial\theta_{lj}},
    \end{align*}
    \begin{equation*}
        \frac{\partial L(k_{2})}{\partial t}=-\sum_{l,j}\left(\frac{\partial L(k_{2})}{\partial\theta_{lj}}\right)^{2}-\sum_{l,j}\frac{\partial L(k_{1})}{\partial\theta_{lj}}\frac{\partial L(k_{2})}{\partial\theta_{lj}},
    \end{equation*}
     and 
    \begin{equation}
        \frac{\partial L}{\partial t}=\frac{\partial\left(L(k_{1})+L(k_{2})\right)}{\partial t}=-\sum_{l,j}\left(\frac{\partial L(k_{1})}{\partial\theta_{lj}}+\frac{\partial L(k_{2})}{\partial\theta_{lj}}\right)^{2}\leq0.\label{eq:losssmall0}
    \end{equation}
    To obtain
    \begin{equation}
        0<\frac{\partial L(k_{1})}{\partial t}-\frac{\partial L(k_{2})}{\partial t}=-\sum_{l,j}\left[\left(\frac{\partial L(k_{1})}{\partial\theta_{lj}}\right)^{2}-\left(\frac{\partial L(k_{2})}{\partial\theta_{lj}}\right)^{2}\right],\label{eq:loss1smallloss2}
    \end{equation}
    it is sufficient to have 
    \begin{equation}
        \left|\frac{\partial L(k_{1})}{\partial\theta_{lj}}\right|>\left|\frac{\partial L(k_{2})}{\partial\theta_{lj}}\right|.\label{eq:lk1largelk2}
    \end{equation}
    Eqs. (\ref{eq:losssmall0}, \ref{eq:loss1smallloss2}) also yield
    to
    \begin{equation*}
        \frac{\partial L(k_{1})}{\partial t}<0.
    \end{equation*}
    Therefore, Eq. (\ref{eq:lk1largelk2}) is a sufficient condition for
    $\mathcal{S}$. Based on the theorem 1, we have proved the theorem
    2.
\end{proof}

\section{Memorizing $2$-d image\label{subsec:2d} }

We train a DNN to fit a natural image (See Fig.~\ref{fig:2dImg}(a)),
a mapping from coordinate $(x,y)$ to gray scale strength, where the
latter is subtracted by its mean and then normalized by the maximal
absolute value. First, we initialize DNN parameters by a Gaussian
distribution with mean $0$ and standard deviation $0.08$ (initialization
with small parameters). From the snapshots during the training process,
we can see that the DNN captures the image from coarse-grained low
frequencies to detailed high frequencies (Fig.~\ref{fig:2dImg}(b)).
As an illustration of the F-Principle, we study the Fourier transform
of the image with respect to $x$ for a fixed $y$ (red dashed line
in Fig.~\ref{fig:2dImg}(a), denoted as the target function $f(x)$
in the spatial domain). The DNN can well capture this 1-d slice after
training as shown in Fig.~\ref{fig:2dImg}(c). Fig.~\ref{fig:2dImg}(d)
displays the amplitudes $|\hat{f}(k)|$ of the first $40$ frequency
components. Due to the small initial parameters, as an example in
Fig.~\ref{fig:2dImg}(d), when the DNN is fitting low-frequency components,
high frequencies stay relatively small. As the relative error shown
in Fig.~\ref{fig:2dImg}(e), the first five frequency peaks converge
from low to high in order.

Next, we initialize DNN parameters by a Gaussian distribution with
mean $0$ and standard deviation $1$ (initialization with large parameters).
After training, the DNN can well capture the training data, as shown
in the left in Fig.~\ref{fig:2dImg}(f). However, the DNN output
at the test pixels are very noisy, as shown in the right in Fig.~\ref{fig:2dImg}(f).
For the pixels at the red dashed lines in Fig.~\ref{fig:2dImg}(a),
as shown in Fig.~\ref{fig:2dImg}(g), the DNN output fluctuates a
lot. Compared with the case of small initial parameters, as shown
in Fig.~\ref{fig:2dImg}(h), the convergence order of the first five
frequency peaks do not have a clear order. 
\begin{center}
\begin{figure}
\begin{centering}
\includegraphics[scale=0.45]{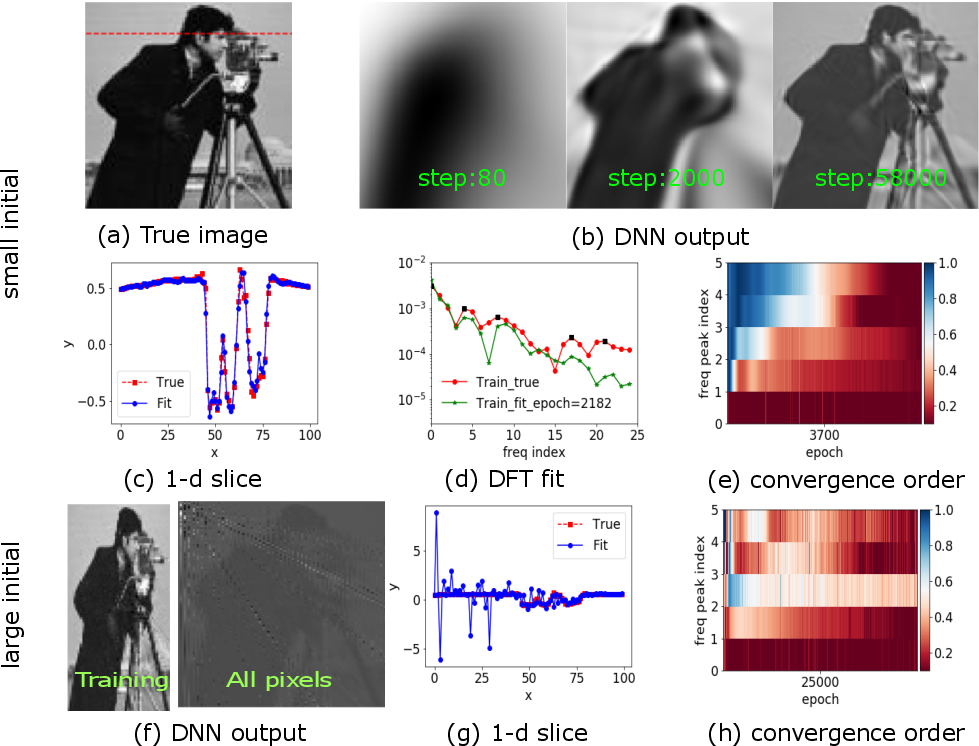} 
\par\end{centering}
\caption{F-Principle in fitting a natural image. The training data are all
pixels whose horizontal indices are odd. We initialize DNN parameters
by a Gaussian distribution with mean $0$ and standard deviation $0.08$
(small initial) or $1$ (large initial). (a) True image. (b-g) correspond
to the case of the small initial parameters. (f-h) correspond to the
case of the large initial parameters. (b) DNN outputs of all pixels
at different training epochs. (c, g) DNN outputs (blue) and the true
gray-scale (red) of test pixels at the red dashed position in (a).
(d) $|\hat{h}(k)|$ (green) at certain training epoch and $|\hat{f}(k)|$
(red) at the red dashed position in (a), as a function of frequency
index. Selected peaks are marked by black dots. (e, h) $\Delta_{F}(k)$
computed by the training data at different epochs for the selected
frequencies in (d). (f) DNN outputs of training pixels (left) and
all pixels (right) after training. We use a tanh-DNN with widths $2$-$400$-$200$-$100$-$1$.
We train the DNN with the full batch and learning rate $0.0002$.
The DNN is trained by Adam optimizer \citep{kingma2014adam} with
the MSE loss function. \label{fig:2dImg}}
\end{figure}
\par\end{center}

\section{Another viewpoint of examining F-Principle in MNIST/CIFAR10 through
filtering method \label{sec:Filter-1}}

The section we present another viewpoint of examining F-Principle
in MNIST/CIFAR10 through filtering method. For readers' convenience,
we describe the filtering method again in this section.

We decompose the frequency space into two domains by a constant $k_{0}$,
i.e., a low-frequency domain of $|\vec{k}|\leq k_{0}$ and a high-frequency
domain of $|\vec{k}|>k_{0}$, where $|\cdot|$ is the length of
a vector. Then, $\vec{y}_{i}$ can be decomposed by $\vec{y}_{i}=\vec{y}_{i}^{\mathrm{low},k_{0}}+\vec{y}_{i}^{\mathrm{high},k_{0}}$,
where $\vec{y}_{i}^{\mathrm{low},k_{0}}$and $\vec{y}_{i}^{\mathrm{high},k_{0}}$
are the low- and high- frequency part of $\vec{y}_{i}$, respectively.
For illustration, $y_{i}=\sin(2\pi x_{i})+\sin(2\pi(2x_{i}))+\sin(3\pi(3x_{i}))$,
if $k_{0}=1.5$, then, $y_{i}^{\mathrm{low},k_{0}}=\sin(2\pi x_{i})$
and $y_{i}^{\mathrm{high},k_{0}}=\sin(2\pi(2x_{i}))+\sin(3\pi(3x_{i}))$.

The DNN is trained as usual by the original dataset $\{(\vec{x}_{i},\vec{y}_{i})\}_{i=0}^{n-1}$.
During the training, we examine the distance between the DNN output
and the low-frequency part of $\{(\vec{x}_{i},\vec{y}_{i})\}_{i=0}^{n-1}$
by MSE, i.e., $\mathrm{Dist}(\vec{y}^{\mathrm{low},k_{0}},\vec{h})$. Under the F-Principle,
the DNN training would be dominated by the low-frequency part $\vec{y}^{\mathrm{low},k_{0}}$
at the early stage, therefore, $\mathrm{Dist}(\vec{y}^{\mathrm{low},k_{0}},\vec{h})$
would decrease. At the latter stage, the training would be dominated
by high-frequency part $\vec{y}_{i}^{\mathrm{high},k_{0}}$, which
would make the the DNN output deviate from the low-frequency part
$\vec{y}^{\mathrm{low},k_{0}}$, therefore, $\mathrm{Dist}(\vec{y}^{\mathrm{low},k_{0}},\vec{h})$
would increase. In short, F-Principle predicts that during the training
of original dataset $\{(\vec{x}_{i},\vec{y}_{i})\}_{i=0}^{n-1}$,
$\mathrm{Dist}(\vec{y}^{\mathrm{low},k_{0}},\vec{h})$ would first decrease and then
increase.

We refer to the \emph{turning epoch of training} when $\mathrm{Dist}(\vec{y}^{\mathrm{low},k_{0}},\vec{h})$
attains its minimum, denoted by $T_{k_{0}}$. If $k_{1}>k_{0}$, then,
$\vec{y}^{\mathrm{low},k_{1}}$ preserves not only all frequency components
in $\vec{y}^{\mathrm{low},k_{0}}$ but also those between $k_{0}$
and $k_{1}$. The DNN would spend more time to converge all frequency
components of $\vec{y}^{\mathrm{low},k_{1}}$ compared with $\vec{y}^{\mathrm{low},k_{0}}$.
Therefore, $T_{k_{1}}>T_{k_{0}}$. In short, F-Principle predicts
that during the training of original dataset $\{(\vec{x}_{i},\vec{y}_{i})\}_{i=0}^{n-1}$,
$T_{k_{0}}$ monotonically increases with $k_{0}$.

Note that$\mathrm{Dist}(\vec{y}^{\mathrm{low},k_{0}},\vec{h})$ is \emph{not} a generalization
error \emph{nor} a test error because $\mathrm{Dist}(\vec{y}^{\mathrm{low},k_{0}},\vec{h})$
quantifies how well DNN learns the low frequency part of the training
dataset. $\vec{y}^{\mathrm{low},k_{0}}$ can be obtained by the following
filtering method.

\subsection{Filtering method \label{subsec:Filtering-method-1}}

We obtain the low-frequency part $\vec{y}^{\mathrm{low},k_{0}}$ by
convolving the original dataset $\{(\vec{x}_{i},\vec{y}_{i})\}_{i=0}^{n-1}$
with a Gaussian filter. An intuition of why such operation can eliminate
the high-frequency part of the original dataset is as follows. The
convolution in the spatial domain is equivalent to the product in
the frequency domain. The Fourier transform of a Gaussian kernel is
still a Gaussian kernel. Therefore, the convolution result is the
product of a Gaussian kernel with the Fourier transform of the original
dataset. Since the tail of a Gaussian kernel exponentially decays,
high-frequency components after filtering almost vanish due to the
product with small numbers (close to zero).

We train the DNN with\emph{ original dataset $\{(\vec{x}_{i},\vec{y}_{i})\}_{i=0}^{n-1}$.
}The\emph{ Gaussian-filtered dataset $\{(\vec{x}_{i},\vec{y}_{i}^{\delta})\}_{i=0}^{n-1}$
used to examine the DNN at each training epoch} can be obtained by
\[
\vec{y}_{i}^{\delta}=\frac{1}{C_{i}}\sum_{j=0}^{n-1}\vec{y}_{j}\exp\left(-|\vec{x}_{i}-\vec{x}_{j}|^{2}/(2\delta)\right),
\]
where $C_{i}=\sum_{j=0}^{n-1}\exp\left(-|\vec{x}_{i}-\vec{x}_{j}|^{2}/(2\delta)\right)$.
For fixed $\delta$, clearly, $\{\vec{y}_{i}^{\delta}\}$ preserves
the low frequency part while losing the high frequency part of $\{\vec{y}_{i}\}$.
When $\delta\rightarrow0$, $\vec{y}_{i}^{\delta}\rightarrow\vec{y}_{i}$,
i.e., keeping all frequencies. When $\delta\rightarrow\infty$, $\vec{y}_{i}^{\delta}\rightarrow\frac{1}{n}\sum_{j=0}^{n-1}\vec{y}_{j}$,
i.e., keeping only the lowest (zero) frequency. As $\delta$ increases,
$\{\vec{y}_{i}^{\delta}\}$ preserves less low-frequency components.
The turning epoch, $T_{\delta}$, would then decrease with $\delta$.
Therefore, the F-Principle predicts that \emph{during the training
of original dataset $\{(\vec{x}_{i},\vec{y}_{i})\}_{i=0}^{n-1}$:}

\emph{First, for a fixed $\delta$, the distance between the DNN output
and a low-frequency part of $\{(\vec{x}_{i},\vec{y}_{i})\}_{i=0}^{n-1}$,
$D(\vec{y}^{\delta},\vec{h})$, would first decrease and then
increase.}

\emph{Second, $T_{\delta}$ monotonically decreases with $\delta$. }

Note that if an algorithm captures the target function from high to
low frequency, these two predictions fail. Ideal experiments are shown
in Appendix \ref{sec:Filterappend} to illustrate this point. 
\begin{center}
\begin{figure}[h]
\begin{centering}
\subfloat[Fitting in Fourier domain ]{\begin{centering}
\includegraphics[scale=0.3]{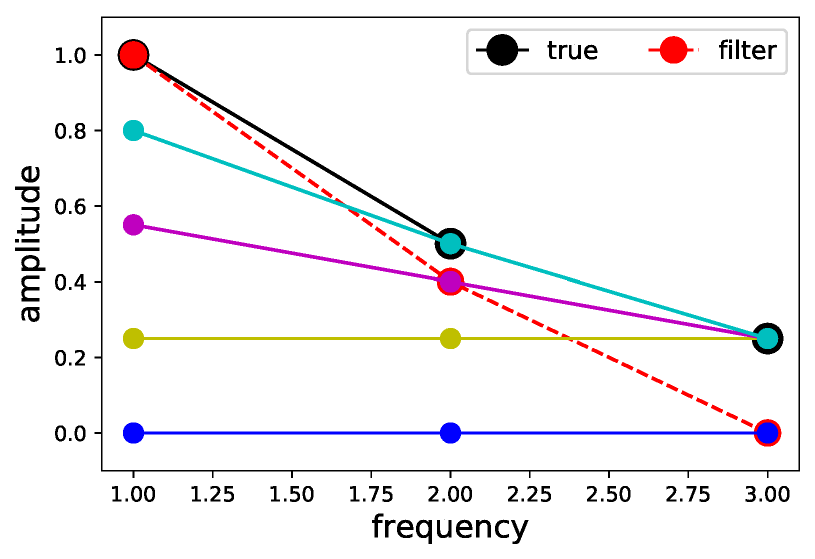} 
\par\end{centering}
}\subfloat[$D(\vec{y}^{\delta},h)$ vs. fitting epoch]{\begin{centering}
\includegraphics[scale=0.3]{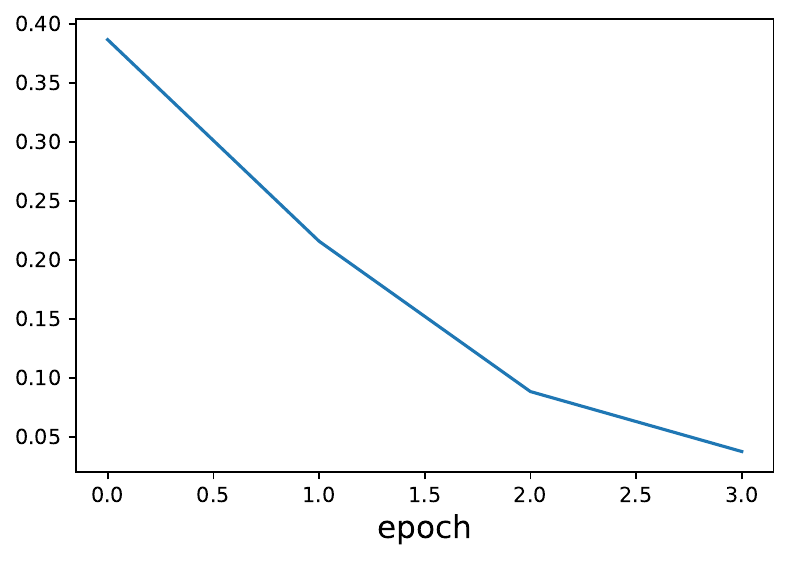} 
\par\end{centering}
}
\par\end{centering}
\caption{Ideal experiment. In (a), the blue, yellow, magenta and cyan curves
correspond to fitting curves at the initial, first, second and third
fitting epoch. \label{fig:ideal-2} }
\end{figure}
\par\end{center}

\subsection{DNNs with various settings}

With the filtering method, we show the F-Principle in the DNN training
process of real datasets. For MNIST, we use a fully-connected tanh-DNN
(no softmax) with MSE loss; for CIFAR10, we use cross-entropy loss
and a ReLU-CNN, followed by a fully-connected DNN with a softmax.

As an example, results of each dataset for one $\delta$ are shown
in Fig. \ref{fig:Noisefitting-Mnist-2}(a). In both cases, $D(\vec{y}^{\delta},\vec{h})$
first decreases and then increases, which meet the first prediction.

As shown in Fig. \ref{fig:Noisefitting-Mnist-2}(b), $T_{\delta}$
monotonically decreases with $\delta$, which meets the second prediction.
We also remark that, based on above results on cross-entropy loss,
the F-Principle is not limited to MSE loss, which possesses a natural
Fourier domain interpretation by the Parseval's theorem as illustrated
in \citep{xu_training_2018} and \citep{rahaman2018spectral}. Note
that the above results holds for optimization methods of both gradient
descent and stochastic gradient descent. 
\begin{center}
\begin{figure}[h]
\begin{centering}
\subfloat[$D(\vec{y}^{\delta},h)$ ]{\begin{centering}
\includegraphics[scale=0.32]{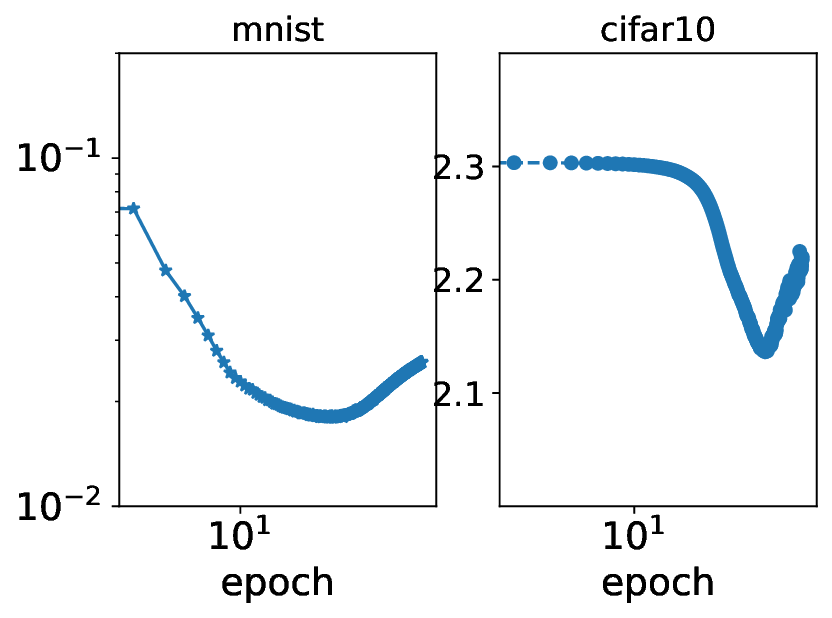} 
\par\end{centering}
}\subfloat[Normalized $T_{\delta}$]{\begin{centering}
\includegraphics[scale=0.32]{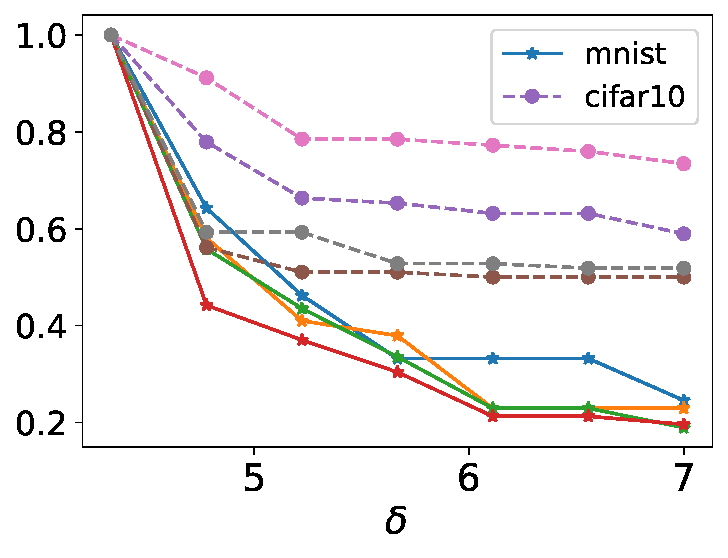} 
\par\end{centering}
}
\par\end{centering}
\caption{F-Principle in MNIST and CIFAR10. (a) $D(\vec{y}^{\delta},\vec{h})$
against training epoch. $\delta=7$ and $4.3$ for MNIST (left) and
CIFAR10 (right), respectively. (b) $T_{\delta}$, normalized by the
maximal $T_{\delta}$ of each trial, is plotted against filter width
$\delta$. Solid and dashed curves are for MNIST and CIFAR10, respectively.
Each curve is for one trial. \label{fig:Noisefitting-Mnist-2} }
\end{figure}
\par\end{center}

\subsection{F-Principle in VGG16 }

It is important to verify the F-Principle in a commonly used and large
DNNs. Therefore, we use the filtering method to show the F-Principle
in the VGG16 \citep{simonyan2014very} equipped with a 1024 fully-connected
layer. We train the network with CIFAR10 from \emph{scratch}. As shown
in Fig. \ref{fig:VGG-1} (a) and (b), the phenomena are consistent
with the first and second prediction in Section \ref{subsec:Filtering-method-1}. 
\begin{center}
\begin{figure}[h]
\begin{centering}
\subfloat[$D(\vec{y}^{\delta},h)$ ]{\begin{centering}
\includegraphics[scale=0.23]{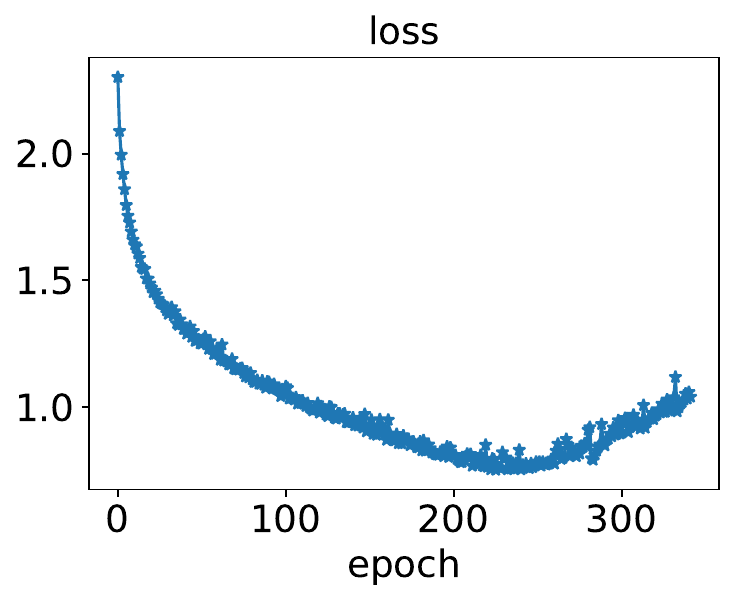} 
\par\end{centering}
}\subfloat[Normalized $T_{\delta}$]{\begin{centering}
\includegraphics[scale=0.23]{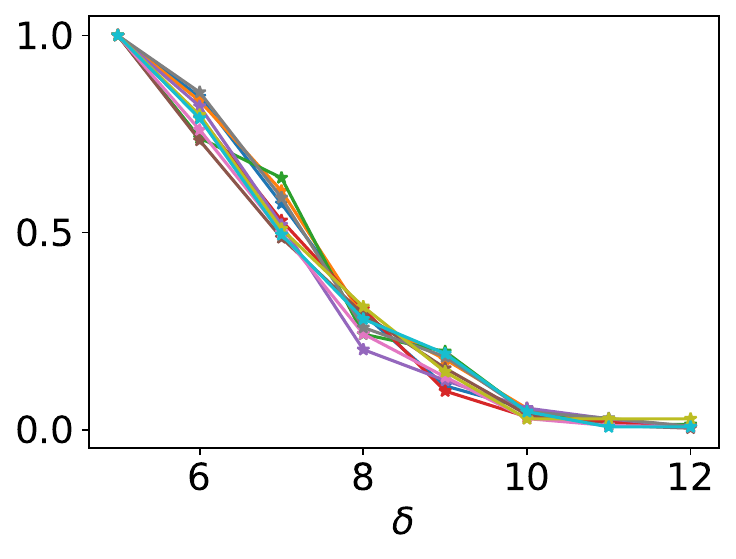} 
\par\end{centering}
}
\par\end{centering}
\caption{F-Principle in VGG16 with CIFAR10. (a) $D(\vec{y}^{\delta},\vec{h})$
against training epoch. $\delta=4$. (b) $T_{\delta}$, normalized
by the maximal $T_{\delta}$ of each trial, is plotted against filter
width $\delta$. Each curve is for one trial. \label{fig:VGG-1} }
\end{figure}
\par\end{center}

\section{Behavior of anti-F-Principle in Synthetic data through filtering
method \label{sec:Filterappend}}

\subsection{Experiments}

Consider a target function is 
\begin{equation*}
    f(x)=c_{0}+\sum_{k=1}c_{k}\sin(2k-1)x.
\end{equation*}
The fitting 
\begin{equation*}
    h(x,t)=c_{0}(1-\exp(-a_{0}t))+\sum_{k=1}(1-\exp(-a_{k}t))c_{k}\sin(2k-1)x.
\end{equation*}

In this section, we refer to F-Principle (anti-F-Principle) if $a_{k}$monotonically
decreases (increases) as $k$, i.e., low (high) frequency has higher
priority when $h(x,t)$ converges to $f(x)$ as $t\rightarrow\infty$.
Fig. \ref{fig:Noisefitting-Mnist-1} shows that F-Principle and anti-F-Principle
have different behavior in the filtering method. In any case, the
F-Principle meets the two predictions in the main text, that is, \emph{during
the training of original dataset $\{(\vec{x}_{i},\vec{y}_{i})\}_{i=0}^{n-1}$:}

\emph{First, for a fixed $\delta$, the distance between the DNN output
and a low-frequency part of $\{(\vec{x}_{i},\vec{y}_{i})\}_{i=0}^{n-1}$,
$D(\vec{y}^{\delta},h)$, would first decrease and then increase. }

\emph{Second, $T_{\delta}$ monotonically decreases with $\delta$. }

An intuitive understanding of the anti-F-Principle is in the next
sub-section. 
\begin{center}
\begin{figure*}
\begin{centering}
\subfloat[]{\begin{centering}
\includegraphics[scale=0.24]{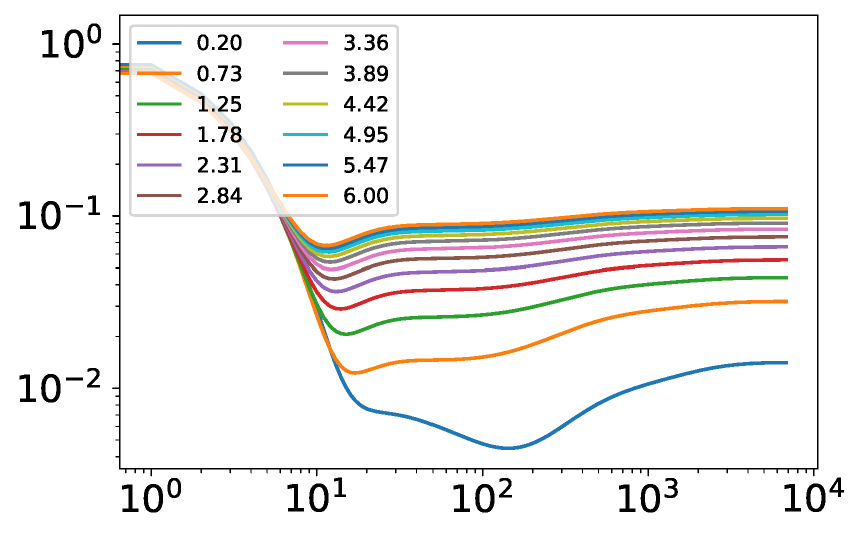} 
\par\end{centering}
}\subfloat[ ]{\begin{centering}
\includegraphics[scale=0.24]{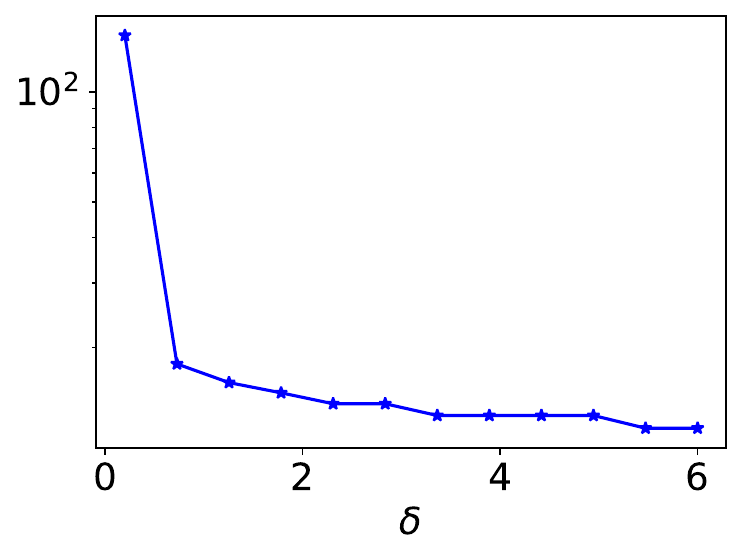} 
\par\end{centering}
}\subfloat[]{\begin{centering}
\includegraphics[scale=0.24]{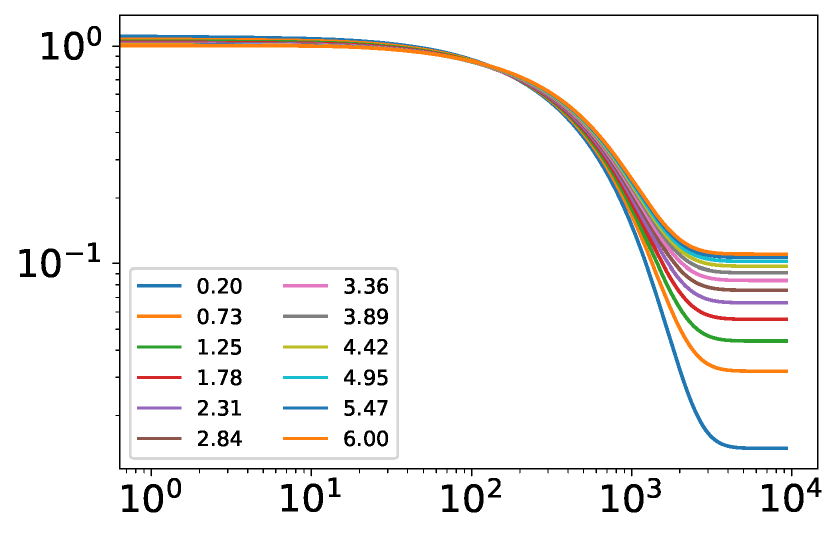} 
\par\end{centering}
}
\par\end{centering}
\begin{centering}
\subfloat[]{\begin{centering}
\includegraphics[scale=0.24]{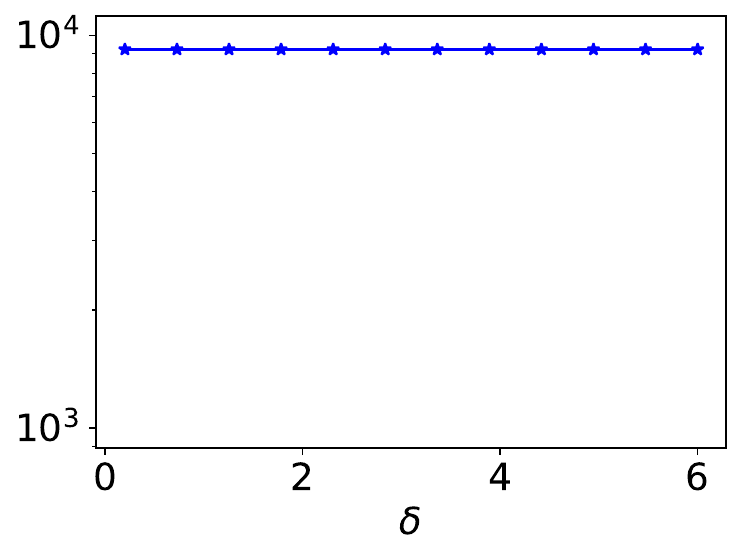} 
\par\end{centering}
}\subfloat[]{\begin{centering}
\includegraphics[scale=0.24]{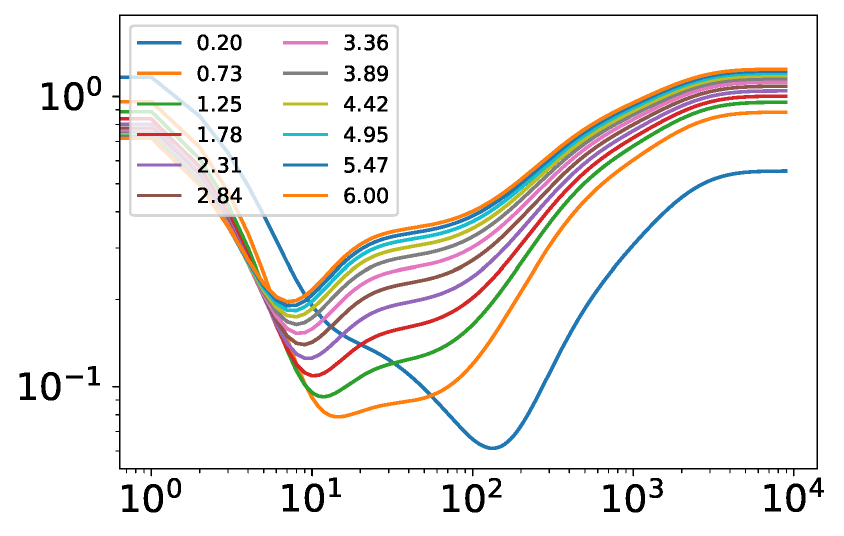} 
\par\end{centering}
}\subfloat[]{\begin{centering}
\includegraphics[scale=0.24]{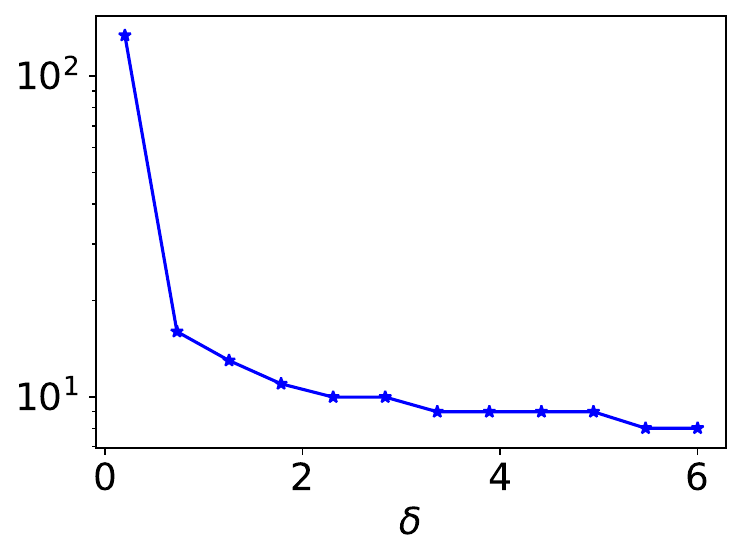} 
\par\end{centering}
}
\par\end{centering}
\begin{centering}
\subfloat[]{\begin{centering}
\includegraphics[scale=0.24]{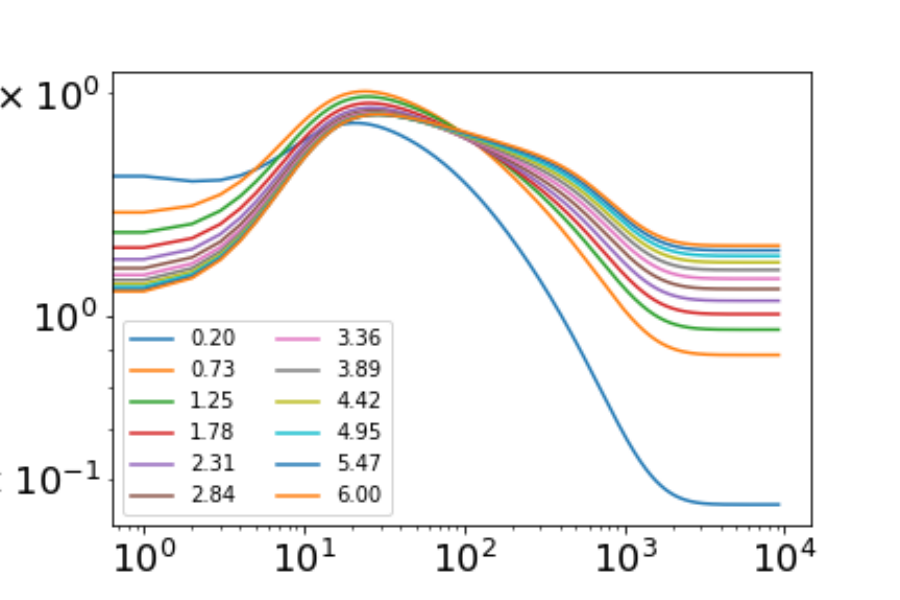} 
\par\end{centering}
}\subfloat[]{\begin{centering}
\includegraphics[scale=0.24]{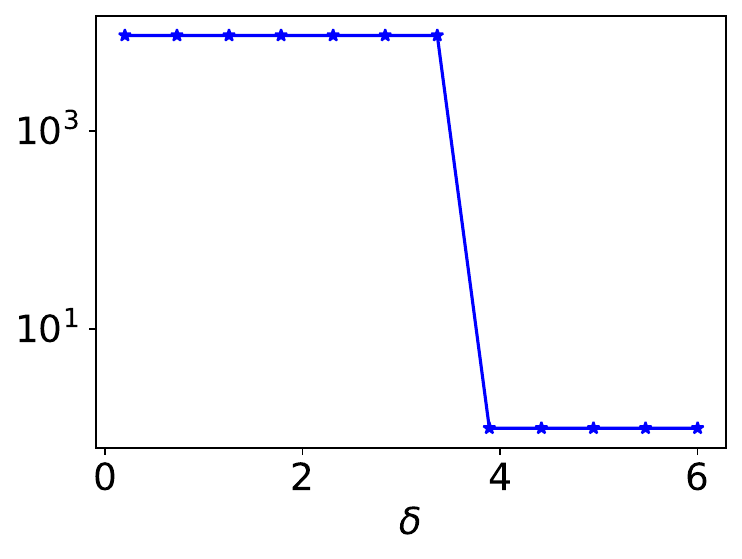} 
\par\end{centering}
}\subfloat[$D(\vec{y}^{\delta},h)$ vs. $t$ ]{\begin{centering}
\includegraphics[scale=0.24]{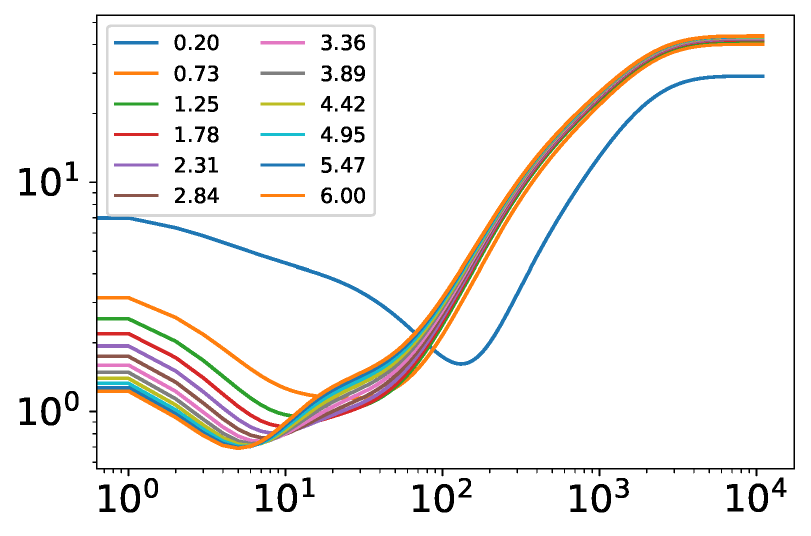} 
\par\end{centering}
}
\par\end{centering}
\begin{centering}
\subfloat[$T_{\delta}$]{\begin{centering}
\includegraphics[scale=0.24]{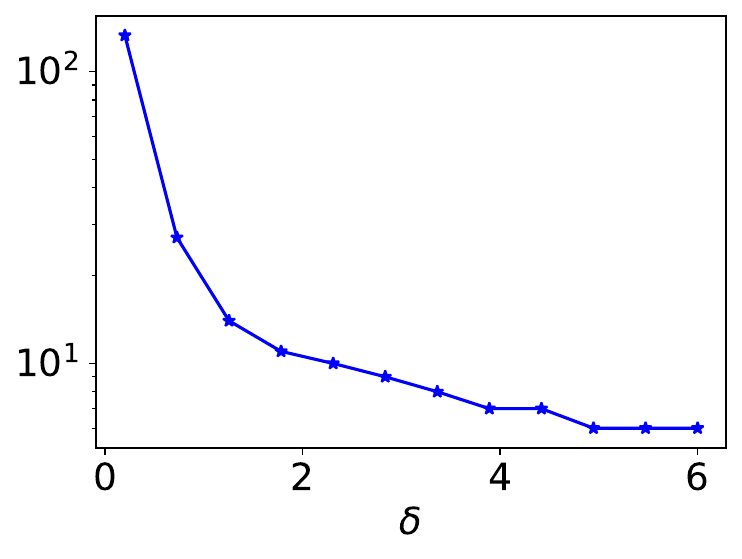} 
\par\end{centering}
}\subfloat[$D(\vec{y}^{\delta},h)$ vs. $t$ ]{\begin{centering}
\includegraphics[scale=0.24]{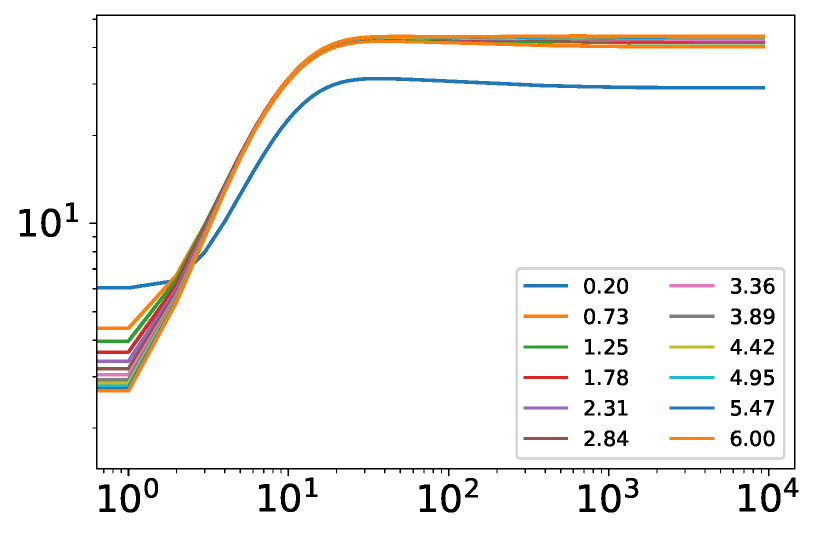} 
\par\end{centering}
}\subfloat[ $T_{\delta}$]{\begin{centering}
\includegraphics[scale=0.24]{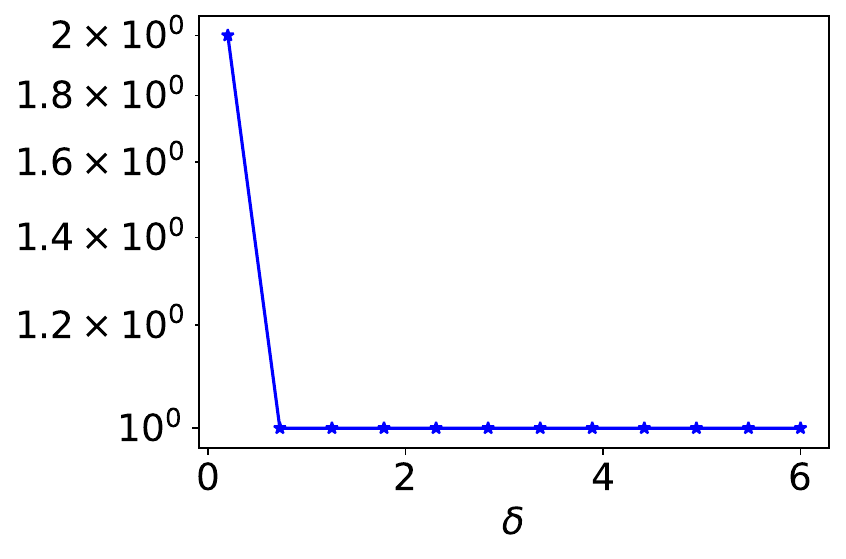} 
\par\end{centering}
}
\par\end{centering}
\caption{The priority coefficient $a_{k}\in[200,150,5,1]$ . Left two columns
(F-Principle): lower frequency has higher priority; right two columns
(anti-F-Principle): higher frequency has higher priority. The amplitude
for each frequency, $c_{k},$ from low to high is {[}1,1/2,1/5,1/8{]},
{[}1,1,1,1{]}, {[}1,2,5,8{]} for three rows, respectively. In the
first and the third column, the legend indicates $\delta$. $x$ is
evenly sampled from $[-6.28,6.28]$ with size $100$. The output and
the loss are computed every $dt=0.001$.\label{fig:Noisefitting-Mnist-1} }
\end{figure*}
\par\end{center}

\subsection{Understanding of the phenomenon of anti-F-Principle}

If an algorithm captures the target function from high to low frequency,
the two F-Principle predictions fail. Here show ideal experiments
in Fig. \ref{fig:ideal-1} to illustrate this claim.

Firstly, consider that the target function decays in Fourier domain.
In Fig. \ref{fig:ideal-1}a, the target function has three frequencies,
marked by black dots; the filtered data $\vec{y}^{\delta}$, marked
by red dots and dashed line, preserves all of the first frequency,
most of the second frequency, and non the third frequency. The initial
value is zero at all three frequencies (Note that the initial DNN
output is often close to zero). As high frequency converges faster,
the third frequency converges while other two frequencies do not,
denoted by the yellow curve. Despite the fitting curve deviates from
the filtered data at the third frequency, due to the large amplitude
of first two frequencies, the combined contribution from the first
two frequencies leads the fitting curve closer to the filtered data,
i.e., $D(\vec{y}^{\delta},h)$ decreases. As the fitting curve
evolves from the magenta one to the cyan one, the fitting curve deviates
from the filtered data at the second frequency, however, it gets much
closer to the filtered data at the first frequency, leading to decrement
of $D(\vec{y}^{\delta},h)$. Therefore, $D(\vec{y}^{\delta},h)$
decreases during the training (Fig. \ref{fig:ideal-1}b), i.e., contradicting
to the F-Principle's predictions fail.

Secondly, consider that the target function keeps constant in Fourier
domain. At a early stage, the fitting curve evolves from the blue
solid one to the blue dashed one, the summation of the first two frequencies,
which makes the fitting curve closer to the filtered one (red dashed),
is large than the change of the third frequency, then, $D(\vec{y}^{\delta},h)$
decreases. At the second stage, the fitting curve evolves from the
magenta solid one to the magenta dashed one, the summation of the
second and the third frequencies, which makes the fitting curve deviate
from the filtered one (red dashed), is large than the change of the
first frequency, then, $D(\vec{y}^{\delta},h)$ increases. At final
stage, only the first frequency, which is shared with the filtered
data, does not converge yet. Then, converging the first frequency
makes $D(\vec{y}^{\delta},h)$ decreases. Therefore, $D(\vec{y}^{\delta},h)$
first decreases, then increases, and finally decreases during the
training (Fig. \ref{fig:ideal-1}d), i.e., contradicting to the F-Principle's
predictions fail.

Thirdly, consider that the target function increases in Fourier domain.
Before the final stage, the third frequency dominates the evolution
(Fig. \ref{fig:ideal-1}e), thus, $D(\vec{y}^{\delta},h)$ increases.
At final stage, $D(\vec{y}^{\delta},h)$ increases slowly due to
the converging of the first frequency, which is small-amplitude. Therefore,
$D(\vec{y}^{\delta},h)$ almost always increases during the training
(Fig. \ref{fig:ideal-1}f), i.e., contradicting to the F-Principle's
predictions fail. 
\begin{center}
\begin{figure}[h]
\begin{centering}
\subfloat[]{\begin{centering}
\includegraphics[scale=0.3]{pic/syntheticfilter/understand/toy_fit_low} 
\par\end{centering}
}\subfloat[]{\begin{centering}
\includegraphics[scale=0.3]{pic/syntheticfilter/understand/toy_loss_low} 
\par\end{centering}
}
\par\end{centering}
\begin{centering}
\subfloat[]{\begin{centering}
\includegraphics[scale=0.3]{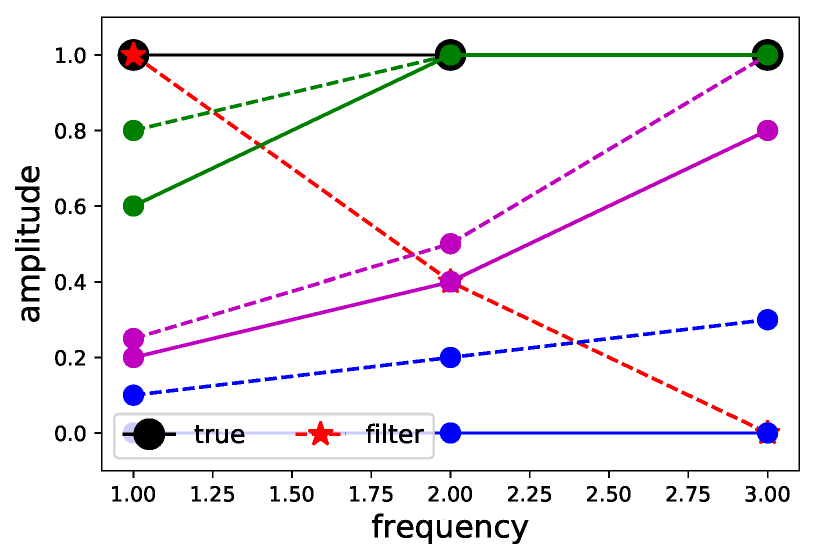} 
\par\end{centering}
}\subfloat[]{\begin{centering}
\includegraphics[scale=0.3]{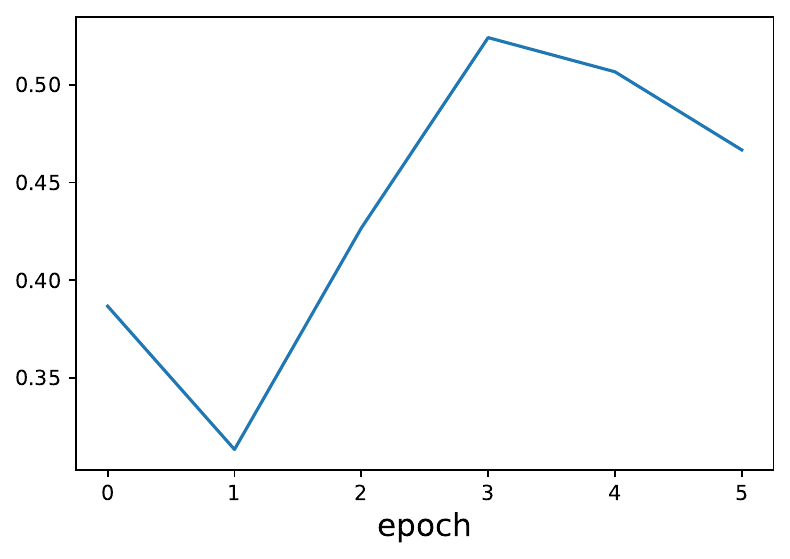} 
\par\end{centering}
}
\par\end{centering}
\begin{centering}
\subfloat[Fitting in Fourier domain ]{\begin{centering}
\includegraphics[scale=0.3]{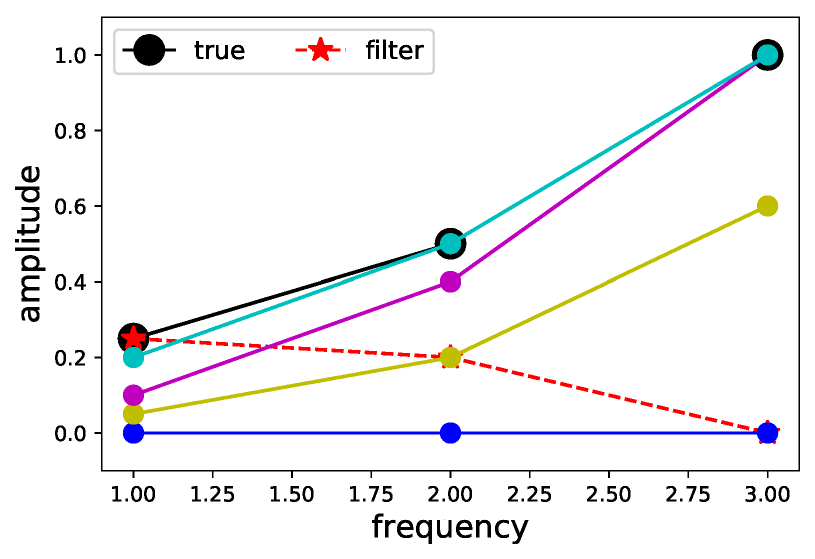} 
\par\end{centering}
}\subfloat[$D(\vec{y}^{\delta},h)$ vs. training epoch]{\begin{centering}
\includegraphics[scale=0.3]{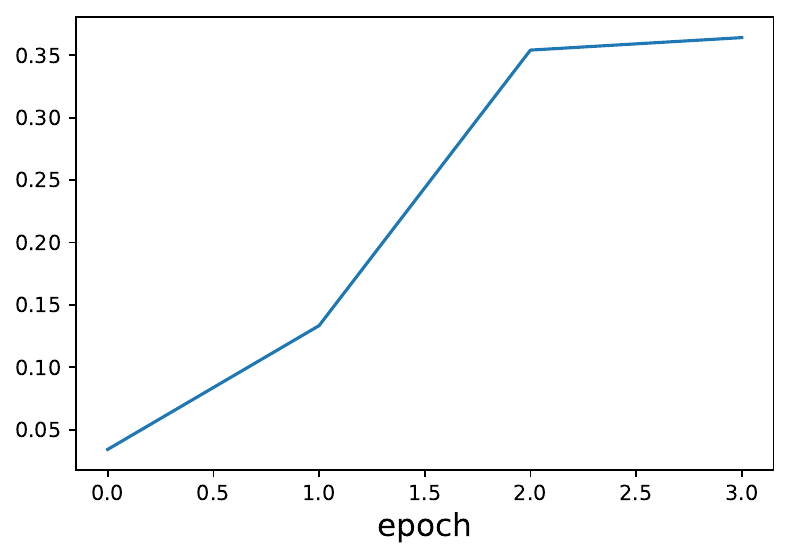} 
\par\end{centering}
}
\par\end{centering}
\caption{Ideal experiment. Each row is for one target function, indicated by
black dots. In the first column, the curves with legends are fitting
curve at certain training epoch. A lower one is from a earlier training
epoch.\label{fig:ideal-1} }
\end{figure}
\par\end{center}
\end{document}